\pdfoutput=1
\documentclass[a4paper]{article}
\usepackage{amsmath,amsfonts,amsthm}
\usepackage{amssymb,latexsym}
\usepackage{graphicx}
\usepackage{mathrsfs}
\usepackage[numbers,sort&compress]{natbib}
\usepackage[latin1]{inputenc}
\usepackage{extarrows}
\theoremstyle{plain}
\newtheorem{theorem}{Theorem}[section]

\newtheorem{proposition}[theorem]{Proposition}

\newtheorem{definition}[theorem]{Definition}
\newtheorem{example}[theorem]{Example}

\newtheorem{notation}[theorem]{Notation}
\newtheorem{remark}[theorem]{Remark}
\numberwithin{equation}{subsection}
\begin{document}

\title{An extension of process calculus for asynchronous communications between agents with epistemic states
%\footnote{This work received financial support of the National Natural Science Foundation of China (Nos. 60973045, 11426136) and Fok Ying-Tung Education Foundation.}
}

\author{Huili Xing
%Zhaohui Zhu$^1$ \footnote{Corresponding author. Email: zhaohui@nuaa.edu.cn} \;
%Jinjin Zhang$^2$  \;
%Yong Zhou$^1$
%\\
( College of Computer Science and Technology\\Nanjing University of Aeronautics and Astronautics; \\Department of Medical Information\\Binzhou Medical University)
%\\
%2 College of Information Science \\ Nanjing Audit University
}

%\date{\today}

\maketitle

%\begin{spacing}{1.0}

\begin{abstract}
  It plays a central role in intelligent agent systems to model agent's epistemic state and its change. Asynchrony plays a key role in distributed systems, in which the messages transmitted may not be received instantly by the agents. To characterize asynchronous communications, asynchronous announcement logic (AAL) has been presented, which focuses on the logic laws of the change of epistemic state after receiving information.
 However AAL does not involve the interactive behaviours between an agent and its environment.
Through enriching the well-known $\pi$-calculus by adding the operators for passing basic facts and applying  the well-known action model logic to describe agents' epistemic states, this paper presents the e-calculus  to model epistemic interactions between agents with epistemic states.
The e-calculus can be adopted to characterize synchronous and asynchronous communications between agents. To capture the asynchrony, a buffer pools is constructed to store the basic facts announced and each agent reads these facts from this buffer pool in some order. Based on the transmission of  link names, the e-calculus is able to realize reading from this buffer pool in different orders.  This paper gives two examples: one is to read in the order in which the announced basic facts are sent (First-in-first-out, FIFO), and the other is in an arbitrary order.

\end{abstract}

\section{Introduction}
It plays a central role in intelligent agent systems to model agent's epistemic attitude and its change.
Multi-agent systems are widely used in many fields, such as robot rescuing~\protect\citep{pettinati2019nmultiagentsystemrobotteams}, intelligent machine community~\protect\citep{seuken2008formalmodelsandalgorithmsunderuncertainty}, multi-UAV cooperation~\citep{harikumar2019multiUAV}, satellite detection~\citep{Azizi2019flyingsatellites}, target monitoring~\citep{Hu2020distributedadaptivetime}, etc.
Asynchrony plays a key role in distributed systems, in which there may be an unpredictable \emph{delay} between sending and receiving messages. In~\citep{Ditmarsch2017asynchronousannouncement},
Hans van Ditmarsch presented
%P. Balbiani and S. Knight et al.
Asynchronous announcement logic (AAL), which is obtained from the standard modal logic by adding two \emph{different} modal operators for sending and receiving messages, and~\citep{Balbiani2021asynchronousannouncement} and~\citep{Knight2017reasoningaboutknowledgeinasynchronoussystem} also explored AAL. The three assumed that messages are emitted publicly by an external and omniscient source and received individually by the agents in the order in which they are sent (First-in-first-out, FIFO). As usual, they also adopted the principle that an agent should know that a message is true before this agent announces this message. AAL is a variant of the well-known Public
announcement logic (PAL)~\citep{Baltag1998PAandCKandPS,Plaza2007publiccommunicationlogic}, where an announcement $\varphi$ is made to a group of agents and then becomes a common knowledge~\citep{Fagin2003reasoningaboutknowledge} in this group (each agent in this group knows $\varphi$, each agent in this group knows that each agent in this group knows $\varphi$, etc.). In PAL, messages are \emph{immediately} received at the same time they are sent.

Both PAL and AAL are devoted to formalizing the change of epistemic state after receiving information. In the family of such logics, Dynamic epistemic logic (DEL)~\citep{Ditmarsch2007dynamicepistemiclogic} is a more powerful one, which can describe more complex forms of communication, such as semi-private announcements, private announcements, announcements with suspicions, etc. Some extensions of this logic are presented to deal with special communication scenarios, e.g., Logic of communication and change~\citep{Benthem2006logicofccommunicationsandchange}, Group announcement logic~\citep{Agotnes2010groupannouncement}, Arbitrary action model logic~\citep{French2014composablelanguageforactionmodels,Hales2013arbitraryactionmodellogic} and some other logics~\citep{Enesser2020Gamedescriptionlanguageanddynamicepistemiclogiccompared,Renne2016Logicsoftemporalepistemicactions,Ditmarsch1999Thelogicofknowledgegamesshowingacard,Ditmarsch2005Dynamicepistemiclogicwithassignment,Benevides2020DynamicEpistemicLogicwithAssignmentsConcurrencyandCommunicationActions,Ditmarsch2009simulationandinformation,Ditmarsch2010Futureeventlogicaxioms,Balbiani2008publicannoucement}.
%This has been widely explored and various logics have been presented. Epistemic modal logic gives a formal account of the epistemic attitude that agents may have, which involves the notions like knowledge, belief and common knowledge, etc, see, e.g,~\cite{Fagin2003reasoningaboutknowledge}.
 All these %dynamic epistemic logics
 consider different communication patterns individually and focus on the logic laws that characterize epistemic actions and the epistemic effect caused by information transmissions.
 %However, in many application scenarios, different communication patterns often coexist.

To model epistemic changes in concurrent situations, based on the well-known $\pi$-calculus, this paper presents a process calculus oriented to epistemic interactions (the e-calculus, for short).  The $\pi$-calculus presented by R. Milner~\citep{Milner1989picalculus,Milner1992picalculus} is a theory of mobile systems whose components can communicate and change their structures. But the $\pi$-calculus does not refer to epistemic information. Different communication patterns may be illustrated easily in the $\pi$-calculus, which will help to characterize epistemic interactions of different communication types. The e-calculus provides a framework to consider how interactions between agents and the environment affect epistemic states. To this end, a finite number of agents form an e-system, in which each agent is related to a process that controls its behaviour and which is always arranged to run at an epistemic state. We apply the well-known Action model logic (AML)~\citep{Ditmarsch2007dynamicepistemiclogic,French2014composablelanguageforactionmodels} to describe agents' epistemic states.
%the e-calculus adopts a notion framework (epistemic environment) to represent epistemic states of agents and operators on epistemic states abstractly. Thus, more space may be left to the users of the e-calculus in specific application situations.

In this paper, we intend to apply the e-calculus to characterize asynchrony. Since an announcement may not be received instantly by the agents in asynchronous communications, it is needed to consider how to store temporarily the messages which have been sent and are waiting to be received and in what order the agents will receive these messages. This is a central issue. In the e-calculus,
on the one hand, the control processes of each agent can transmit the link names like the usual $\pi$-processes, which will help to enable the links between agents to be created and disappear. On the other hand, these processes  can transmit basic facts, which may change the epistemic states of the related agents.
Due to these,  to capture the \emph{asynchrony}, this paper will construct a buffer pool to store the announced basic facts and each agent will share it. In the e-calculus, its storage cells  can be read in different orders.

This paper is organized as follows. We firstly recall the well-known $\pi$-calculus in the next section. Section 3 introduces the e-calculus. Section 4  gives several useful properties about epistemic interactions under AML. In Section 5 and Section 6, we apply the e-calculus to formalize synchronous and asynchronous communications respectively. Finally we end the paper with a brief discussion in Section 7.
%Let $Agent$ be a finite set of agents, and let $Atom$ and $Name$ be a countable set of proposition letters and names of communication channels respectively.
\section{Preliminaries}\label{sec:preliminary}
\subsection{The $\pi$-calculus}\label{subsec:pi-calculus}
The well-known $\pi$-calculus was proposed by R. Milner~\citep{Milner1989picalculus,Milner1992picalculus}. This subsection will briefly recall its primary notions according to the literature~\citep{Sangiorgi2001pi-calculus}.

Let $Name$ be a countable infinite set of names, which can be thought of as
communication channels, locations, considered objects, etc. We use $ a,b,c,d,e,f,$ $g,h,x,y,z$ to range over $Name$.
\begin{definition}[Prefix]\label{def:prefix}
In the $\pi$-calculus, the capabilities for actions are expressed via the \emph{prefixes}  generated by, where $a,c,z \in Name$,
\[ \pi ::=\overline{a}c \mid a(z) \mid \tau \]
\end{definition}
\noindent The capability $\overline{a}c$ is to send the name
$c$  via the channel $a$,  $a(z)$  to receive any name  via the channel $a$, and $\tau$ is a capability for an unobservable action.
\begin{definition}[\citep{Sangiorgi2001pi-calculus}]\label{def:pi-calculus}
The \emph{process} terms in the $\pi$-calculus are generated by the BNF grammar below, where $z \in Name$,
\[\begin{array}{lll}
P &::=& SUM \mid (P \mid P) \mid (\nu z\, P )\mid \;( !\, P)  \\
SUM &::=&\mathbf{0} \mid (\pi.\,P ) \mid (SUM + SUM).
\end{array}\]
%We often abbreviate $\nu z_1 \cdots \nu z_n$ to ($\nu z_1, \cdots, z_n)$.
\end{definition}
\noindent We write $\mathcal{P}$ for the set of all processes, which is ranged over by $P,Q,X,Y,Z$. 
%Here we briefly interpret the above operators. As usual, $\mathbf{0}$ is \emph{inaction}, expressing a ``dead'' process. ~$\mid\;$ is the \emph{composition} operator over processes. ~$\nu z$ is the \emph{restriction} operator which binds the scope of the name $z$. The \emph{prefix} $\pi.\,P$ has a single capability expressed by $\pi$ and $P$ can not proceed until this capability has been exercised. The \emph{summation} $P+Q$ is to evolve according to $P$ (or $Q$) with the other being rendered void. The \emph{replication} $!\,P$, an infinite composition $P\mid P \mid \cdots$, is used to express infinite behaviours.

%\begin{definition}[Free name and bound name~\citep{Sangiorgi2001pi-calculus}]\label{def:free and bound name}
In processes of the forms $a(z).\,Q$ and $\nu z\, Q$, the name $z$ is bound with the scope $Q$. An occurrence of a name in a process is \emph{bound} if it lies within the scope of a binding occurrence of this name, and otherwise \emph{free}. 
%\end{definition}
%\begin{notation}\label{not-notation3}
We use the notation $fn(\mathbf{E})$ (or, $bn(\mathbf{E})$) to denote the set of all \emph{free} (\emph{bound}, resp.) names which occur in  the  entity $\mathbf{E}$ (e.g., process, action).

In the $\pi$-calculus, two process are said to be $\,\alpha$-\emph{convertible} if one may be gotten from another one by changing  \emph{bound} names. As usual, two $\alpha$-convertible processes are identified, and the entities of interest are the equivalence classes of processes  modulo the $\alpha$-convertibility.
%\begin{convention}[Operator precedence]\label{con:operator precedence}
%In a process, prefixing, restriction and replication bind more tightly than composition, and prefixing more tightly than sum.
%Thus, some parentheses may be omitted if no ambiguity.
%\end{convention}
\emph{Structural Operational Semantics} (SOS) was proposed by G. Plotkin~\citep{Plotkin2004astructuralapproachtooperationalsemantics}, which adopts a syntax oriented view on operational semantics and gives operational semantics
in logical style.
We below recall the SOS rules of the $\pi$-calculus which describe the \emph{transition relations} between process terms.

%In the $\pi$-calculus, the prefixes prescribe four kinds of actions:
\begin{definition}[Action~\citep{Sangiorgi2001pi-calculus}]\label{def:action-p}
The \emph{actions} in the $\pi$-calculus are given by, with $a,c,z \in Name$,
\[ \alpha ::=\overline{a}c  \mid ac \mid \overline{a}(z) \mid \tau. \]
\end{definition}
\noindent We write $ Act$ for the set of all actions and use $\alpha,\beta,\gamma$ to range over actions.
The action $\overline{a}c$ ($ac$) is to send (receive, resp.) the name $c$ via the channel $a$, the \textit{bound-output} action $\overline{a}(z)$ is to send the bound name $z$ via  $a$, and $\tau$ is an invisible action.

Table~\ref{Ta2:SOS of process}~\citep{Sangiorgi2001pi-calculus} provides the SOS rules of the $\pi$-calculus, which  captures the capability of processes. There are four rules being elided from Table~\ref{Ta2:SOS of process}: the symmetric form (SUM-R) of (SUM-L), which has $Q+P$ in place of $P+Q$, and the symmetric forms (PAR-R), (CLOSE-R) and (COMM-R) of (PAR-L), (CLOSE-L) and (COMM-L), in which the roles of the left and right components are swapped.
The reader may refer to~\citep{Sangiorgi2001pi-calculus}  for their informal interpretations.
\begin{table}[h]
\caption{The SOS rules for processes}
\label{Ta2:SOS of process}
\setlength{\tabcolsep}{3pt}
%\centering
\begin{tabular}{|p{337pt}|}
\hline
$\text{ }$\\
 $\begin{array}{llll}
 %  \displaystyle (P_{\pi}) \; \; \text{ the rules as in $\pi$-calculus, but with } \alpha \in Act & \; \\
  % &\\
   (\text{OUT-name}) & \frac{ }{\overline{a} c.\,P \stackrel{\overline{a} c}{\longrightarrow}_p P} &  \qquad \;  (\text{TAU}) & \frac{ }{\tau.\,P \stackrel{\tau}{\longrightarrow}_p P }\\
   &&&\\
   (\text{IN-name}) & \frac{ }{a(z).\,P \stackrel{ac}{\longrightarrow}_p P \{c/z\}} &  \qquad \;(\text{SUM-L}) & \frac{P \stackrel{\alpha}{\longrightarrow}_p P'}{P+Q \stackrel{\alpha}{\longrightarrow}_p P'} \\
   &&&
   \end{array}$
   \noindent $\begin{array}{lll}
   (\text{PAR-L}) & \frac{P \stackrel{\beta}{\longrightarrow}_p P'}{P \mid Q \stackrel{\beta}{\longrightarrow}_p P' \mid Q}  & \quad bn(\beta) \cap fn(Q)= \emptyset \\
   &&\\
   (\text{RES}) & \frac{P \stackrel{\beta}{\longrightarrow}_p P'}{\nu z\, P \stackrel{\beta}{\longrightarrow}_p \nu z\, P'} & \quad z \notin na(\beta)  \\
   &&\\
   (\text{OPEN})&  \frac{P \stackrel{\overline{a} z}{\longrightarrow}_p P'}{\nu z\, P \stackrel{\overline{a} (z)}{\longrightarrow}_p P'} &\quad z \neq a\\
   %   &&\;
   \end{array}$

   \noindent $\begin{array}{lll}
&&\\
   (\text{CLOSE-L}) & \frac{P \stackrel{\overline{a} (z)}{\longrightarrow}_p P' \quad\quad Q \stackrel{a z}{\longrightarrow}_p Q'}{P\mid Q \stackrel{\tau}{\longrightarrow}_p \nu z\, (P' \mid Q')} &\quad z \notin fn(Q) \\
   &&\\
  (\text{COMM-L}) & \frac{P \stackrel{\overline{a} c}{\longrightarrow}_p P' \quad\quad Q \stackrel{a c\;\;}{\longrightarrow_p} Q'}{P \mid Q \stackrel{\tau\;\;}{\longrightarrow_p} P' \mid Q'} &\\
  &&\\
    (\text{REP-ACT})& \frac{P \stackrel{\beta}{\longrightarrow}_p P'}{!\,P \stackrel{\beta}{\longrightarrow}_p P' \mid\, !\,P}   \\
&&\\
(\text{REP-COMM}) &\frac{P \stackrel{\overline{a} c\;\;}{\longrightarrow_p} P' \qquad P \stackrel{a c}{\longrightarrow}_p P''}{!\,P \stackrel{\tau}{\longrightarrow}_p (P' \mid P'') \mid \,!\,P}&\\
   &&\\
   (\text{REP-CLOSE})& \frac{P \stackrel{\overline{a} (z)}{\longrightarrow}_p P' \quad\quad P \stackrel{az}{\longrightarrow}_p P''}{!\,P \stackrel{\tau}{\longrightarrow}_p (\nu z\,(P' \mid P'')) \mid\, !\,P} &\quad z \notin fn(P)
 %  &&\;
   %\displaystyle (\text{P-LEARN}) \; \frac{}{L_{\mathcal{B}}^{\mathcal{C}}.\; P  \stackrel{L_{\mathcal{B}}^{\mathcal{C}}}{\longrightarrow} P} &\\
  \end{array}$\\
  $\text{ }$\\
\hline
\end{tabular}
%\label{tab1}
\end{table}

%A \emph{literal} is a positive literal or negative literal.
%\emph{Positive literals} are all expressions of the form $Q \stackrel{\beta}{\longrightarrow}_p Q'$,
%while \emph{negative literals} are all expressions of the form $Q\; \;/\kern-1.3em\xlongrightarrow{\beta}_p$.
%\begin{definition}[Proof tree of processes]\label{def:prooftree}
%A positive literal $ Q \stackrel{\beta}{\longrightarrow}_p Q'$ is \emph{inferred} if there exists a well-founded, upwardly branching tree, whose nodes are labelled by positive literals, such that\\
%  $\bullet$ the root is labelled with $ Q \stackrel{\beta}{\longrightarrow}_p Q'$;\\
% $\bullet$ if a positive literal $\delta$ is the label of a node and $\{\delta_k:k\in I \}$ is the set of labels \\
%%\begin{adjustwidth}{0.8em}{0em}
  %$\text{ }$ of the nodes directly above this node, then there exists a SOS rule $\frac{\{\vartheta_k\,:\,k\in I \}}{\vartheta} $\\
%  $\text{ }$ and a substitution $\rho$ which
 %   satisfies
%$\delta=\rho(\vartheta) $ and
 %$\delta_k=\rho(\vartheta_k) $ for every $k\in I$.
 %%\end{adjustwidth}
 %\end{definition}
 %If $ Q \stackrel{\beta}{\longrightarrow}_p Q'$ is inferred  then  $Q'$ is said to be an $\beta$-derivative of $Q $, or the
%process $Q$ is said to evolve into $Q'$ after executing the action $\beta$.
%The assertion $Q\; \,\;/\kern-1.3em\xlongrightarrow{\beta}$ holds if $Q$ has no $\beta$-derivative.
\begin{definition}[LTS of processes~\citep{Sangiorgi2001pi-calculus}]\label{def:LTSp}
The \emph{labelled transition system (LTS) of processes} is a triple
$\langle \mathcal{P}, Act, \longrightarrow_p \rangle$, in which
$\longrightarrow_p \subseteq \mathcal{P} \times Act \times \mathcal{P}$
%a pair $\langle \mathcal{P}, \{\stackrel{\alpha}{\longrightarrow}_p \}_{\alpha\in Act}\rangle$, where $\{\stackrel{\alpha}{\longrightarrow}_p\}_{\alpha\in Act}\subseteq \mathcal{P}\times \mathcal{P}$
is the transition relation defined by
$ \longrightarrow_p \triangleq \{\langle Q, \beta,Q'\rangle :\,  Q \stackrel{\beta}{\longrightarrow}_p Q' \text{ has a proof tree}\} $.
% by the SOS rules in Table~\ref{Ta2:SOS of process}.
\end{definition}
%We write $P \stackrel{\alpha}{\longrightarrow}_p P'$ if $\,\langle P,\; \alpha, \; P'\rangle \in \; \longrightarrow_p$.
\subsection{Action model logic}\label{subsec:AML}
In this subsection, we will recall the language $\mathcal{L}_{am}$ of the well-known action model logic (AML)~\citep{Ditmarsch2007dynamicepistemiclogic,French2014composablelanguageforactionmodels} which augments the standard modal language by adding action model operators.
%, and then give some useful properties of this logic. Let $\mathcal{L} $ be a set of propositional letters.
\begin{definition}[Kripke model]\label{def:Kripke model}
A \emph{Kripke model} $M$ is a triple $\langle W^{M},R^{M}, V^{M}\rangle$ where
%\begin{adjustwidth}{0em}{0em}

\noindent $\bullet\;$ $W^{M}$ is a non-empty set of states, which is ranged over by  $s,t,u,v,w$,

\noindent $\bullet\;$ $R^{M}:Agent \rightarrow 2^{W^{M} \times W^{M}}$
is an accessibility function assigning a binary

\noindent $\text{ }\;\,$ relation $R^{M}_C \subseteq W^{M} \times W^{M}$ to each agent $C \in Agent$, and

\noindent $\bullet\;$ $V^{M}:Atom  \rightarrow 2^{W^{M}}$ is a valuation function.
%\end{adjustwidth}

\noindent A pair $(M,w)$ with $w\in W^{M}$ is said to be a pointed Kripke model.
\end{definition}
For any binary relation $R\subseteq W_1\times W_2$ and $w\in W_1$, $R(w)\triangleq\{w'\in W_2 \mid  w R w'\}$.
% and $R^+$  is the transitive closure of $R$.
%We use $\circ$ to denote the composition operator of relations. Given a set $\mathfrak{R}$ of relations, $\mathfrak{R}^n$ is used to denote the composition of $n$ relations in $\mathfrak{R}$ where $n\geq 0 $ and $\mathfrak{R}^0 $ is the identity relation.
 %and if $n=0$ then $R^{\ast} $ is the identity relation.
%%$\mathfrak{R}^{\ast} \triangleq R_1\circ \cdots \circ R_n$ where $n\geq 0 $, $R_i \in \mathfrak{R}$ for each $1\leq i\leq n$
\begin{definition}[Action model]\label{def:Action model}
Let $\mathcal{L}_0$ be a logical language. An \emph{action model} $\verb"M"$ is a triple $\langle \verb"W"^{\verb"M"},\verb"R"^{\verb"M"}, \verb"pre"^{\verb"M"}\rangle$ where

\noindent $\bullet\;$ $\verb"W"^{\verb"M"}$ is a non-empty finite set of action points, which is ranged over by $\mathbf{s},\mathbf{t},\mathbf{u},\mathbf{v},\mathbf{w}$,

%\noindent $\text{ }\;\,$ $\mathbf{s},\mathbf{t},\mathbf{u},\mathbf{v},\mathbf{w}$,

\noindent $\bullet\;$ $\verb"R"^{\verb"M"}:Agent \rightarrow 2^{\verb"W"^{\verb"M"} \times \verb"W"^{\verb"M"}}$
is an accessibility function assigning a binary relation

\noindent $\text{ }\;\,$  $\verb"R"^{\verb"M"}_C \subseteq \verb"W"^{\verb"M"} \times \verb"W"^{\verb"M"}$ to each agent $C \in Agent$, and

\noindent $\bullet\;$ $\verb"pre"^{\verb"M"}:\verb"W"^{\verb"M"} \rightarrow \mathcal{L}_0$ is a precondition function  assigning a precondition

\noindent $\text{ }\;\,$  $\verb"pre"^{\verb"M"} (\mathbf{v}) \in \mathcal{L}_0$ to each action point $\mathbf{v} \in \verb"W"^{\verb"M"}$.
he precondition of an action point
%For each $p\in Atom$, $V^M(p)$ is the set of states in $\mathfrak{M}$ where $p$ is true.

\noindent $\text{ }\;\,$  T is considered as the necessary condition for this action to happen.

\noindent A pointed action model is a pair $(\verb"M",\mathbf{v})$ where $\mathbf{v}\in \verb"W"^{\verb"M"}$ denotes the \emph{actual} action point.
\end{definition}
%The class of all pointed Kripke (action) models is denoted by $\mathcal{K}$ (resp., $\mathcal{AM}$).
Some examples of action models that realize different kinds of behavioural aims may be found in~\citep{Baltag1998PAandCKandPS,Benthem2006logicofccommunicationsandchange,Plaza2007publiccommunicationlogic}.
%\begin{convention}\label{con:action model figure}
%As usual, in drawing  an action (or, Kripke) model, the preconditions of action points (resp., the propositional letters true at states) are indicated on their right, and the actual action point (resp., state) is underlined.
%\end{convention}
\begin{definition}[Language $\mathcal{L}_{am}$~\citep{Ditmarsch2007dynamicepistemiclogic}]\label{def:language of action model logic}
 The language $\mathcal{L}_{am}$ of \emph{action model logic} is generated by BNF grammar as follows.
 \[ \varphi ::=r\mid (\neg\varphi) \mid (\varphi_1 \wedge \varphi_2) \mid (\Box_C \varphi )\mid  ([\,\verb"M",\mathbf{v}\,]\, \varphi). \]
%\[ \varphi ::=p\mid (\neg\varphi) \mid (\varphi_1 \wedge \varphi_2) \mid (\Box_A \varphi )\mid (\mathbf{C}_{\mathcal{B}}\, \varphi)\mid ([\,\verb"M",\mathbf{s}\,]\, \varphi) \]
Here, $r\in Atom $, $C\in Agent$ and $(\verb"M",\mathbf{v}) $ is a pointed action model such that $\verb"pre"^{\verb"M"} (\mathbf{w}) \in \mathcal{L}_{am}$ for each $\mathbf{w} \in \verb"W"^{\verb"M"}$.
%in which the precondition of every action point is an $\mathcal{L}$-formula.
 %that has already been constructed in a previous stage of the inductively defined hierarchy.
 The modal operator $\diamondsuit_A$ and propositional connectives $\vee$, $\rightarrow$, $\leftrightarrow$, $\top$ and $\bot$ are used in the standard manner. Moreover, we often write $\langle\, \verb"M",\mathbf{v}\,\rangle \varphi$ for $\neg [\,\verb"M",\mathbf{v}\,] \neg\varphi$.
\end{definition}
The formula $\Box_C \varphi$ represents that  the agent $C$ \emph{knows} $\varphi$.
%In the action model operator $[\,\verb"M",\mathbf{s}\,]$,  $(\verb"M",\mathbf{s})$ is a syntactic object.
%The readers may refer to~\citep{Ditmarsch2007dynamicepistemiclogic,Baltag1998PAandCKandPS} for the explanations about this in detail.
The formula $[\,\verb"M",\mathbf{v}\,]\, \varphi$ expresses that $\varphi$ holds after executing the action denoted by the action model $(\verb"M",\mathbf{v})$.
%It is obvious that the fragment of $\mathcal{L}_{am}$ which consists of all formulas containing no action model operator is indeed  \emph{the multi-agent modal language}.
\begin{definition}[Semantics~\citep{Ditmarsch2007dynamicepistemiclogic}]\label{def:semantics of action model logic}
 Given a Kripke model $M$, the notion of a formula $\varphi \in \mathcal{L}_{am}$ being satisfied in $M$ at a state $s$ is defined inductively as follows
\[\begin{array}{lll}
  M,s \models r & \text{ iff } & s\in V^{M}(r), \;\text{ where } r \in Atom  \\
  M,s \models \neg \varphi & \text{ iff } & M,s \nvDash \varphi\\
  M,s \models \varphi_1 \wedge \varphi_2  & \text{ iff } & M,s \models \varphi_1 \text{ and }  \,M,s \models \varphi_2 \\
  M,s \models \Box_C \varphi & \text{ iff } & \text{for each } t \in R^{M}_C (s),\;\; M,t \models \varphi\\
 %  \mathfrak{M},s \models_{\mathfrak{C}} \mathbf{C}_{\mathcal{B}}\, \varphi & \text{ iff } & \text{for all } t \in (\bigcup_{A\in \mathcal{B}} R^{\mathfrak{M}}_A)^+ (s),\;\; \mathfrak{M},t \models_{\mathfrak{C}} \varphi\\
  M,s \models [\,\verb"M",\mathbf{v}\,] \varphi & \text{ iff } & M,s \models \verb"pre"^{\verb"M"} (\mathbf{v}) \;\text{ implies }
 M \otimes \verb"M", \langle s,\mathbf{v}\rangle \models \varphi.
\end{array}\]
%$\text{ }$

\noindent Here, we define $\models$ and $(M,s) \otimes (\verb"M",\mathbf{v})$ simultaneously. $(M,s) \otimes (\verb"M",\mathbf{v})$ is the result after executing $(\verb"M",\mathbf{v})$ at $(M,s) $, which is defined as
\[(M,s) \otimes(\verb"M",\mathbf{v})\triangleq  \begin{cases}
     (M \otimes\verb"M", \langle s,\mathbf{v}\rangle)  & \text{    if } \;M,s \models \verb"pre"^{\verb"M"} (\mathbf{v})\\
     ( M,s) & \text{    otherwise},
    \end{cases}\]
%if and only if $\,\mathfrak{M},s \models_{\mathfrak{C}} \verb"pre"^{\verb"M"} (\mathbf{s})$. If $\,\mathfrak{M},s \models_{\mathfrak{C}} \verb"pre"^{\verb"M"} (\mathbf{s})$, we set $(\mathfrak{M},s) \otimes_{\mathfrak{C}} (\verb"M",\mathbf{s})\triangleq (\mathfrak{M} \otimes_{\mathfrak{C}} \verb"M", \langle s,\mathbf{s}\rangle)$,
where $\,M \otimes \verb"M" \triangleq \langle W,R,V\rangle$ is the Kripke model defined by, ~for each $C\in Agent$ and $r\in Atom $,
%\begin{myindentpar}{1.5em}
\[\begin{array}{lll}
 W &\triangleq &\{\langle u,\mathbf{u}\rangle :\, u \in W^{M}, \mathbf{u} \in \verb"W"^{\verb"M"} \text{ and } M,u \models \verb"pre"^{\verb"M"} (\mathbf{u})\}\\
R_C  &\triangleq & \{\langle \langle u,\mathbf{u}\rangle, \,\langle v,\mathbf{v}\rangle\rangle \in W \times W :\,  u R^{M}_C v \text{ and } \mathbf{u} \verb"R"^{\verb"M"}_C \mathbf{v}\,\}\\
V(r) &\triangleq  &\{\langle u,\mathbf{u}\rangle \in W :\,  u \in V^{M} (r)\,\}.
%\langle u,\mathbf{u}\rangle R_A \langle v,\mathbf{v}\rangle & \text{iff}& u R^{\mathfrak{M}}_A v \text{ and } \mathbf{u} \verb"R"^{\verb"M"}_A \mathbf{v}\\
%\text{ } \langle u,\mathbf{u}\rangle \in V(p) &\text{iff}& u \in V^{\mathfrak{M}} (p)
\end{array}\]
%\end{myindentpar}
\end{definition}
%As usual, if $\mathfrak{C}=\mathcal{S}5$, the formula $\Box_A \varphi$ reads ``the agent $A$ \emph{knows} that $\varphi$'' and $\mathbf{C}_{\mathcal{B}} \,\varphi$ stands for ``$\varphi$ is \emph{common knowledge} among $\mathcal{B}$ (every agent in $\mathcal{B}$ knows $\varphi$, every agent in $\mathcal{B}$ knows that every agent in $\mathcal{B}$ knows $\varphi$, etc.);'' if $\mathfrak{C}=\mathcal{K}45$, the formula $\Box_A \varphi$ reads ``the agent $A$ \emph{believes} that $\varphi$'' and $\mathbf{C}_{\mathcal{B}} \,\varphi$ stands for ``$\varphi$ is \emph{common belief} among $\mathcal{B}$ (every agent in $\mathcal{B}$ believes $\varphi$, every agent in $\mathcal{B}$ believes that every agent in $\mathcal{B}$ believes $\varphi$, etc.).''

If $M,s \models\verb"pre"^{\verb"M"} (\mathbf{v})$, $(\verb"M",\mathbf{v})$ is said to be \emph{executable} at $(M,s)$.
As usual, for every $\varphi \in \mathcal{L}_{am}$, $\varphi$ is valid, in symbols $\models\varphi$, if ~$M,s \models \varphi$ for every pointed Kripke model $(M,s)$.

 \begin{definition}[Bisimilarity of Kripke models]\label{def:Kripke model bisimulation}
 Given two models $M=\langle W, R,  V \rangle$ and $M'=\langle W', R',  V' \rangle$, a binary relation $\mathcal{Z} \subseteq W\times W'$ is a \emph{bisimulation}  between $M$ and $M'$ if, for each pair $\langle s,s' \rangle$ in $\mathcal{Z}$ and $C \in Agent$,\\
% \begin{myindentpar}{0em}
 \noindent \textbf{(atoms)} $\text{ } \; s \in V(r) \text{ iff } \,s' \in V'(r) \text{ for each } r\in Atom $;\\
 $\text{\textbf{(forth)}} \quad \text{ for each }  t\in W, \;\, s R_C t \text{ implies }  s' R'_C t' \text{ and } t \mathcal{Z} t'\text {for some } t'\in W'$;\\
 $\text{\textbf{(back)}} \quad \,\text{ for each }  t'\in W', \;\, s' R'_C t' \text{ implies } s R_C t \text{ and } t \mathcal{Z} t'\text{ for some } t\in W$.
%\end{myindentpar}

 \noindent We say that $(M,s)$ and $(M',s')$ are bisimilar, in symbols $(M,s) \underline{\leftrightarrow} (M',s')$,
 if there exists a bisimulation relation between $M$ and $M'$ which links $s$ and $s'$. We often write $\mathcal{Z}: (M,s) \underline{\leftrightarrow}  (M',s')$ to indicate that $\mathcal{Z}$ is a bisimulation relation such that $s \mathcal{Z} s'$.
 \end{definition}
 As usual, the relation $\underline{\leftrightarrow} $ is an equivalence relation. It is well known that $\mathcal{L}_{am}$-satisfiability is invariant under bisimulations~\citep{Ditmarsch2007dynamicepistemiclogic}, formally,
 \begin{proposition}[\citep{Ditmarsch2007dynamicepistemiclogic}]\label{prop:bisi invariance}
  If $(M,s) \underline{\leftrightarrow} (N,t)$ then
  \[M,s\models \varphi\;\; \text{ iff }\; \;\;N,t\models \varphi  \quad \text{ for each } \varphi \in \mathcal{L}_{am}.\]
\end{proposition}

\section{The e-calculus}\label{sec:e-calculus}
This section presents the e-calculus, which enriches the well-known $\pi$-calculus by adding epistemic interactions.
We will give the syntax and SOS rules of the e-calculus.

Let $Agent$ be a finite set of agents, which is ranged over by $ A,B,C,D$. Let $Atom$ be a set of propositional letters (\emph{basic facts}), which is ranged over by $ p,q,r,\chi$.
%We use $ \mathcal{A},\mathcal{B},\mathcal{C},\mathcal{D}$ to range over $2^{Agent}$, whose elements are as usual called \emph{groups}.
%We use $ p,q,r,\chi$ to range over propositional letters (\emph{facts}).
\subsection{Syntax}\label{subsec:e-syntax}
The terms of the e-calculus include processes and e-systems. A finite number of agents form an e-system and each agent in an e-system is associated with a process which controls its behaviour.

The e-calculus augments the prefixes of the $\pi$-calculus by adding the ones for the capabilities of sending and receiving a \emph{basic fact}, which are of the forms: $\overline{a}q$ and $a(\chi)$ with $a \in Name$ and $\chi, q \in Atom$.
%\begin{definition}[Prefix]\label{def:prefix}
%The \emph{prefixes} in the e-calculus are provided by
%\[ \pi ::=\overline{a}c \mid a(z) \mid \tau \mid \overline{a}q \mid a(\chi)\]
%where $a,c,z \in Name$ and $\chi, q \in Atom$.
%\end{definition}
%\noindent Here, 
The capability $\overline{a}q$ is to send the fact $q$ via the channel $a$, and $a(\chi)$ to receive any fact via the channel $a$. Except that these new prefixes are introduced,
the \emph{process terms} in the e-calculus coincide with the ones in the $\pi$-calculus. 
%In a process of the form $a(\chi).\,Q$, the \emph{propositional letter} $\chi$ is \emph{bound} with the scope $Q$, which describes the \emph{places} in $Q$ where the propositional letter received via the channel $a$ will be substituted when this  process acts.
\begin{definition}\label{def:e-calculus}
The \emph{pseudo e-systems} are defined by the BNF grammar below
\[ H ::=([\,P\,]_A) \mid (H\parallel H) \mid  (\nu z\, H) \]
where $A \in Agent$, $z \in Name$ and $P$ is a \emph{process term}. An \emph{e-system} in the e-calculus is a pseudo e-system in which any agent  occurs at most once.
%\emph{In a system, any agent occurs at most one time}.
\end{definition}
\noindent We write $\mathcal{S}$ for the set of all e-systems and use $G$ and $H $ to range over e-systems.
Here, the e-system $[\,P\,]_A$ is composed of only one \emph{agent} $A$, which is related with the (\emph{control}) process $P$. The symbols ~$\parallel$ and $\nu$ express the \emph{composition} and \emph{restriction} operators over e-systems respectively. Their functions analogize  those of the corresponding operators on  processes. In particular, for the e-system of the form $\nu z\, H$, the name $z$ is \emph{bound} with the scope $H$.

For e-systems,  bound names (propositional letters) and the $\alpha$-convertibility w.r.t.
 bound names and propositional letters can be analogously defined as the corresponding notions for processes.
\subsection{Actions, labels and epistemic states}\label{subsec:epistemic actions}
%\begin{definition}[Action]\label{def:action}
%The \emph{actions} in the e-calculus are given by
%\[ \alpha ::=\overline{a}c  \mid ac \mid \overline{a}(z) \mid \tau \mid \overline{a}q \mid aq \]
%where $a,c,z \in Name$ and $q \in Atom$.
%\end{definition}
%\noindent We write $ Act$ for the set of all actions and use $\alpha,\beta,\gamma$ to range over actions.
Compared with the $\pi$-calculus, the e-calculus adds two actions for transmitting basic facts, which are of the forms $\overline{a}q$ and $aq$ with $a\in Name$ and $q \in Atom$. 
%\noindent Here, 
The action $\overline{a}q$ ($aq$) is to send (receive, resp.) the fact $q$ via the channel $a$, and the other actions have the same meaning as the corresponding ones in the $\pi$-calculus.
\begin{definition}[Label]\label{def:action}
The set of all \emph{labels} in the e-calculus, denoted by $Label$, is defined as
\begin{multline*}
Label \;\;\triangleq \;\;\{\overline{a}c, \,\overline{a}(z),\, ac,\,\tau \,:\, a,c,z \in Name \}\;\cup\\
\{\langle \overline{a}q,A \rangle,\, \langle aq,A \rangle,\,\langle a,q,A,B\rangle\,:\, \begin{subarray}{1} a\in Name,\,A,B\in Agent,\\A\neq B \text{ and } q\in Atom \end{subarray}\,\}.
%Label^{\circledast} &\triangleq &\{\langle \tau,x,q,A,\mathcal{B}\rangle\, :\; \langle \tau,x,q,A,\mathcal{B}\rangle \in Label \}.
\end{multline*}
\end{definition}
\noindent We use $l$  to range over labels.
The label $\langle \overline{a} q,A\rangle $ ($\langle a q,A\rangle $) means that the action $\overline{a}q$ ($aq$, resp.) is executed by the agent $A$, while the label $\langle a,q,A,B\rangle $ describes the epistemic interaction  in which the agent $A$ sends the fact $q$ to the agent $B$ via the channel $a$.

We apply the well-known Action model logic (AML) to illustrate agents' epistemic environment, in which, as usual, we use Kripke  models to represent  epistemic states and construct action models to describe epistemic actions. Thus, applying such action models on epistemic states (Kripke models) through the operator $\otimes$ (see Definition~\ref{def:semantics of action model logic}), the corresponding epistemic actions will cause the evolution of epistemic states of agents.

%\begin{proof}
 % By structural induction on $\psi$.
  % referring to~\citep[Theorem 5.20]{Ditmarsch2007dynamicepistemiclogic}.
%\end{proof}

Now we define the action models for receiving and interacting  basic facts as follows.
\begin{definition}[Action model for receiving facts]\label{def: K-plus-action model}
Let $B\in Agent$ and $q\in Atom$. The action model $\verb"M"^{q,B} $  with $\mathbf{s}$ as the actual action point is defined as $\verb"M"^{q,B} \triangleq \langle\, \{\mathbf{s},\mathbf{s}_\mathbf{q},\mathbf{s}_\mathbf{\top}\},\verb"R", \verb"pre"\,\rangle$ with
$\verb"pre"(\mathbf{s})= \verb"pre"(\mathbf{s}_\mathbf{\top})\triangleq \top$, $\verb"pre"(\mathbf{s}_\mathbf{q}) \triangleq q$,
$\verb"R"_B \triangleq \{\langle \mathbf{s},\mathbf{s}_\mathbf{q}\rangle, \langle \mathbf{s}_\mathbf{q},\mathbf{s}_\mathbf{\top}\rangle, \langle \mathbf{s}_\mathbf{\top}, \mathbf{s}_\mathbf{\top}\rangle  \}$, and $\verb"R"_C \triangleq \{\langle \mathbf{s},\mathbf{s}_\mathbf{\top}\rangle, \langle \mathbf{s}_\mathbf{q},\mathbf{s}_\mathbf{\top}\rangle, \langle \mathbf{s}_\mathbf{\top}, \mathbf{s}_\mathbf{\top}\rangle  \}$ for each $C\in Agent-\{ B\}$. See Fig.~\ref{figure: ek1}.
% Graphically,
%See Figure~\ref{figure: ek1}.
%\(\begin{array}{lll}
%\verb"W"^{\verb"M"^{\oplus_{aq,B}}_{\mathcal{K}} } & \triangleq &\{\mathbf{s},\mathbf{s}_\mathbf{q},\mathbf{s}_\mathbf{T}\}\\
%\verb"R"^{\verb"M"^{\oplus_{aq,B}}_{\mathcal{K}} }_D &\triangleq &
%\begin{cases}
 % R\{\langle \mathbf{s},\mathbf{s}\rangle, \langle \mathbf{s}_\mathbf{T}, \mathbf{s}_\mathbf{T}\rangle  \} & \text{ if $D\in \{A,B,C\}$}\\
  %\{\langle \mathbf{s},\mathbf{s}_\mathbf{T}\rangle, \langle \mathbf{s}_\mathbf{T}, \mathbf{s}_\mathbf{T}\rangle  \} & \text{ if $D\in Agent-\{A,B,C\}$}
%\end{cases} \\
%\verb"pre"^{\verb"M"^{\oplus_{aq,B}}_{\mathcal{K}} }(\mathbf{s}_\mathbf{q})& \triangleq & q \quad  \text{ and } \; \verb"pre"^{\verb"M"^{\oplus_{aq,B}}_{\mathcal{K}} }(\mathbf{s})= \verb"pre"^{\verb"M"^{\oplus_{aq,B}}_{\mathcal{K}} }(\mathbf{s}_\mathbf{T})\triangleq \top
%\end{array}\)

\begin{figure}[h]
\setlength{\abovecaptionskip}{0.4cm}
\setlength{\belowcaptionskip}{-0.1cm}
\centering
%\begin{center}
\centerline{\includegraphics[scale=0.8]{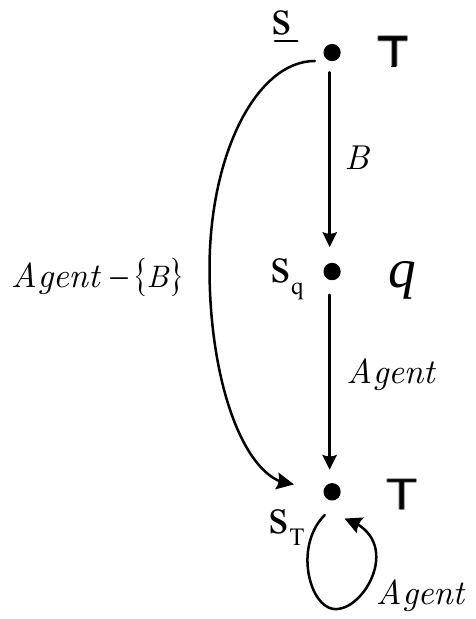}}
%\end{center}
\caption{The action model for the agent $B$ receiving the fact $q$}\label{figure: ek1}
\end{figure}
\end{definition}
%This action model realizes the epistemic aim that the agent $B$ cognises the fact $q$ after receiving it,
%and this action does not affect the other agents, which means that the epistemic information of the other agents is not changed after this action.
\begin{definition}[Action model for interacting facts]\label{def: K-star-action model}
Let $q\in Atom $ and $A,B\in Agent$ with $A\neq B$. The action model $\verb"M"^{q,A,B}$  with $\mathbf{s}$ as the actual action point is defined as
$\verb"M"^{q,A,B} \triangleq \langle\, \{\mathbf{s},\mathbf{s}_\mathbf{q},\mathbf{s}_\mathbf{\top}\},\verb"R", \verb"pre"\,\rangle$ with
$\verb"pre"(\mathbf{s}) \triangleq \Box_A q$, $\verb"pre"(\mathbf{s}_\mathbf{q}) \triangleq q$, $ \verb"pre"(\mathbf{s}_\mathbf{\top})\triangleq \top$, $\,\verb"R"_C \triangleq \{\langle \mathbf{s},\mathbf{s}_\mathbf{q}\rangle, \langle \mathbf{s}_\mathbf{q},\mathbf{s}_\mathbf{q}\rangle, \langle \mathbf{s}_\mathbf{\top}, \mathbf{s}_\mathbf{\top}\rangle  \}$ for each $C\in \{A,B \} $,
and $\verb"R"_C \triangleq \{\langle \mathbf{s},\mathbf{s}_\mathbf{\top}\rangle, \langle \mathbf{s}_\mathbf{q},\mathbf{s}_\mathbf{\top}\rangle, \langle \mathbf{s}_\mathbf{\top}, \mathbf{s}_\mathbf{\top}\rangle  \}$ for each $C\in Agent- \{A,B \}$. See Fig.~\ref{figure: ek2}.
%Graphically,
%See Figure~\ref{figure: ek2}.
\begin{figure}[h]
\setlength{\abovecaptionskip}{0.4cm}
\setlength{\belowcaptionskip}{-0.1cm}
\centering
%\begin{center}
\centerline{\includegraphics[scale=0.8]{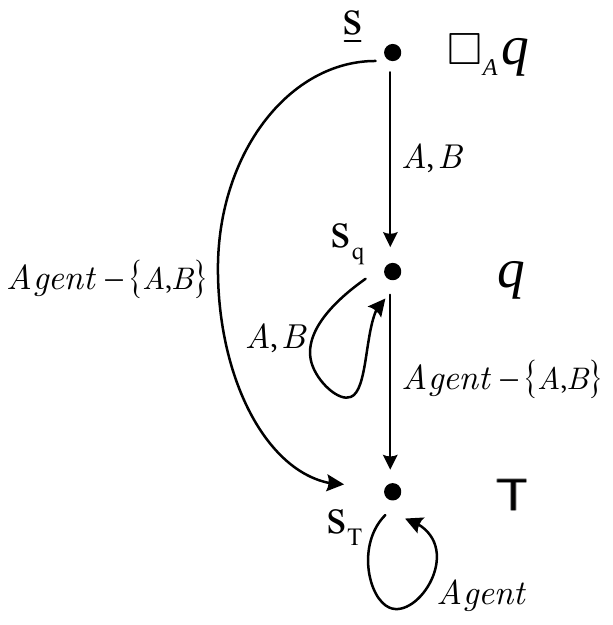}}
%\end{center}
\caption{The action model  for passing the fact $q$ from the agent $A$ to $B$}\label{figure: ek2}
\end{figure}
\end{definition}
%\begin{convention}\label{con:action model notation}
 %For a given  action model  defined in either Definition~\ref{def: K-plus-action model} or~\ref{def: K-star-action model}, if its actual action point is denoted by $ \mathbf{w}$, we often use $\mathbf{w}_\mathbf{\varphi}$ to denote the non-actual action point with the precondition $\varphi$ (Notice that their preconditions are different from each other).
%\end{convention}
\subsection{SOS rules}\label{subsec:SOS}
The SOS rules of the e-calculus are divided into two parts: SOS$_{\text{p}}$ and SOS$_{\text{e}}$. The former captures the capability of \emph{processes} in the e-calculus, which coincides with the ones of the $\pi$-calculus (see, Table~\ref{Ta2:SOS of process}) except adding the rules (OUT-fact) and  (IN-fact) as follows:

$\begin{array}{ll}
 \text{(OUT-fact)} &\quad\frac{ }{\overline{a} q.\,P \stackrel{\overline{a} q}{\longrightarrow}_p P}   \\ \text{(IN-fact)} & \quad \frac{ }{a(\chi).\,P \stackrel{aq}{\longrightarrow}_p P \{q/ \chi\}}
  \end{array}$
%given in Table~\ref{Ta2:SOS of process-a}. 

\noindent The latter describes the behaviours of \emph{e-systems}, which are listed in Table~\ref{Ta1:SOSe}. Similar to Table~\ref{Ta2:SOS of process}, the rules (PAR$_{\text{e}}$-R), (CLOSE$_{\text{e}}$-R), (COMM$_{\text{e}}$-name-R) and (COMM$_{\text{e}}$-fact-R) are elided from Table~\ref{Ta1:SOSe}.

%\begin{table}[h]
%\caption{The SOS rules for processes added in the e-calculus}
%\label{Ta2:SOS of process-a}
%\setlength{\tabcolsep}{3pt}
%%\centering
%\begin{tabular}{|p{337pt}|}
%\hline
%$\text{ }$\\
% $\begin{array}{ll}
 %\text{(OUT-fact)} \quad\frac{ }{\overline{a} q.\,P \stackrel{\overline{a} q}{\longrightarrow}_p P}   & \qquad\qquad \quad\text{(IN-fact)}  \quad \frac{ }{a(\chi).\,P \stackrel{aq}{\longrightarrow}_p P \{q/ \chi\}}
%  \end{array}$\\
 % $\text{ }$\\
%\hline
%\end{tabular}
%%\label{tab1}
%\end{table}

\begin{table}[h]
\caption{The SOS rules for the e-calculus (SOS$_{\text{e}}$)}
\label{Ta1:SOSe}
\setlength{\tabcolsep}{3pt}
%\centering
\begin{tabular}{|p{337pt}|}
\hline
$\text{ }$\\
 $ \begin{array}{lllll}
    (\pi) & \frac{P \stackrel{\beta}{\longrightarrow}_{p} P'}{(M,s) \blacktriangleright [\,P\,]_A \stackrel{\beta}{\longrightarrow} (M,s) \blacktriangleright [\,P'\,]_A} \quad  \beta\in Act \text{ and } \beta  \neq \overline{a}q, aq  \\
    &&&&\\
     (\text{IN}_{\text{e}}) & \frac{P \stackrel{aq}{\longrightarrow}_{p} P' }{(M,s) \blacktriangleright [\,P\,]_A  \xlongrightarrow {\langle a q,A\rangle } (M,s)\otimes  (\texttt{M}^{q,A},\mathbf{s}) \blacktriangleright [\,P'\,]_A }\\
   &&&&\\
    (\text{OUT}_{\text{e}}) & \frac{ P \stackrel{\overline{a} q}{\longrightarrow}_{p} P'}{(M,s) \blacktriangleright [\,P\,]_A \xlongrightarrow {\langle \overline{a} q,A\rangle} (M,s) \blacktriangleright [\,P'\,]_A} \quad M,s \models \Box_A q\\
     &&&&\\
     (\text{PAR}_{\text{e}}\text{-L}) & \frac{(M,s) \blacktriangleright G \stackrel{l}{\longrightarrow} (M',s') \blacktriangleright G'}{(M,s) \blacktriangleright G \parallel H \stackrel{l}{\longrightarrow} (M',s') \blacktriangleright G' \parallel H} \quad bn(l) \cap fn(H)= \emptyset \\
   &&&&\\
    \end{array}$\\

  % \noindent $ \begin{array}{lll}
   %(\text{IN}_{\text{e}}) & \frac{P \stackrel{aq}{\longrightarrow}_{p} P' }{(M,s) \blacktriangleright [\,P\,]_A  \xlongrightarrow {\langle a q,A\rangle } (M,s)\otimes  (\texttt{M}^{q,A},\mathbf{s}) \blacktriangleright [\,P'\,]_A } & \\
    %&&\\
     %  \end{array}$\\

    \noindent $ \begin{array}{lllll}
     (\text{OPEN}_{\text{e}}) & \frac{(M,s) \blacktriangleright G \stackrel{\overline{a} z}{\longrightarrow} (M,s) \blacktriangleright G'}{(M,s) \blacktriangleright \nu z\, G \stackrel{\overline{a} (z)}{\longrightarrow} (M,s) \blacktriangleright G'} \quad  z \neq a\\
      &&&&\\
     (\text{RES}_{\text{e}}) & \frac{(M,s) \blacktriangleright G \stackrel{l}{\longrightarrow} (M',s') \blacktriangleright G'}{(M,s) \blacktriangleright \nu z\, G \stackrel{l}{\longrightarrow} (M',s') \blacktriangleright \nu z\, G'} \quad  z \notin na(l) \text{ whenever }l\neq \langle a,q,A,B\rangle\\
   &&&&\\
  %    (\text{PAR}_{\text{e}}\text{-L}) & \frac{(M,s) \blacktriangleright G \stackrel{l}{\longrightarrow} (M',s') \blacktriangleright G'}{(M,s) \blacktriangleright G \parallel H \stackrel{l}{\longrightarrow} (M',s') \blacktriangleright G' \parallel H} & bn(l) \cap fn(H)= \emptyset \\
   %&&\\
   \end{array}$\\
%   $\text{ }$\\

%\noindent $ \begin{array}{lll}
 %  (\text{RES}_{\text{e}}) \; \frac{(M,s) \blacktriangleright G \stackrel{l}{\longrightarrow} (M',s') \blacktriangleright G'}{(M,s) \blacktriangleright \nu z\, G \stackrel{l}{\longrightarrow} (M',s') \blacktriangleright \nu z\, G'} \;\;\;  z \notin na(l) \text{ whenever }l\neq \langle a,q,A,B\rangle\\
  % &&\\
   %% \text{OPEN}_{\text{e}} & \frac{M \blacktriangleright G \stackrel{\overline{x} z}{\longrightarrow} M \blacktriangleright G'}{M \blacktriangleright \nu z\, G \stackrel{\overline{x} (z)}{\longrightarrow} M \blacktriangleright G'} &  z \neq x\\
   %%&&\\
   %\end{array}$\\

   \noindent $\begin{array}{ll}
    (\text{CLOSE}_{\text{e}}\text{-L}) & \frac{(M,s) \blacktriangleright G \stackrel{\overline{a} (z)}{\longrightarrow} (M,s) \blacktriangleright G' \quad\quad (M,s) \blacktriangleright H \stackrel{a z}{\longrightarrow} (M,s) \blacktriangleright H'}{(M,s) \blacktriangleright G \parallel H \stackrel{\tau}{\longrightarrow} (M,s) \blacktriangleright \nu z\, (G' \parallel H')} \;\; z \notin fn(H) \\
   &\\
   (\text{COMM}_{\text{e}}\text{-name-L}) & \frac{(M,s) \blacktriangleright G \stackrel{\overline{a} c}{\longrightarrow} (M,s) \blacktriangleright G' \quad\quad (M,s) \blacktriangleright H \stackrel{a c}{\longrightarrow} (M,s) \blacktriangleright H'}{(M,s) \blacktriangleright G \parallel H \stackrel{\tau}{\longrightarrow} (M,s) \blacktriangleright G' \parallel H'} \\
   &\\
   (\text{COMM}_{\text{e}}\text{-fact-L}) &  \frac{(M,s) \blacktriangleright G \xlongrightarrow {\langle\overline{a} q,A\rangle} (M_1,s_1) \blacktriangleright G'\quad\quad (M,s) \blacktriangleright H \xlongrightarrow {\langle a q,B\rangle} (M_2,s_2) \blacktriangleright H'}{(M,s) \blacktriangleright G \parallel H \xlongrightarrow {\langle a,q ,A,B\rangle} (M,s)\otimes (\texttt{M}^{q,A,B},\mathbf{s}) \blacktriangleright G' \parallel H'} \\
   &\\
  \end{array}$

  $\text{ }$\\
\hline
\end{tabular}
%\label{tab1}
\end{table}

In order to describe the effect of epistemic interactions, an e-system is always arranged to run at a given epistemic state when its behaviour is considered.
%This is different from usual process calculus (e.g., CCS, the $\pi$-calculus, ACP, etc).
Naturally, the conclusions of the rules in SOS$_{\text{e}}$ have the following form.
\[ (M,s)\blacktriangleright G  \stackrel{l}{\longrightarrow} (N,t) \blacktriangleright H.\]
Intuitively, it says that, through the $l$-labelled transition, the e-system $G$ at the epistemic state $(M,s)$ may become $H$, and the epistemic state $(M,s)$ also evolves into $(N,t)$.

%Based on the rules in Tables~\ref{Ta2:SOS of process},~\ref{Ta2:SOS of process-a} and~\ref{Ta1:SOSe}, the notion of proof tree of a positive literal like $(M,s)  \blacktriangleright G \stackrel{l}{\longrightarrow} (N,t) \blacktriangleright H$ can be defined in the standard way. Further,

Based on the rules in SOS$_{\text{p}}$ and SOS$_{\text{e}}$, we may define the labelled transition system  of e-systems as Definition~\ref{def:LTSp}. 
Now we explain the rules in Table~\ref{Ta1:SOSe} intuitively. The rule  ($\pi$) reveals that, it only depends on the capability of the control process $P$ (not on the epistemic state $(M,s)$) for the agent $A$ to perform actions not concerned with \emph{basic facts}, and these actions also have no effect on epistemic states.
The rules (PAR$_{\text{e}}$-L), (RES$_{\text{e}}$), (OPEN$_{\text{e}}$), (CLOSE$_{\text{e}}$-L) and (COMM$_{\text{e}}$-name-L) are the variants of the corresponding ones in the $\pi$-calculus, which capture the operators $\parallel $ and $\nu$ over e-systems. From these rules, it is easy to see that the operators $\parallel $ and $\nu$ have the same operational semantic as the operators $| $ and $\nu$ respectively except that the former are applied over e-systems instead of processes.
The rule (OUT$_{\text{e}}$) means that an agent $A$ may send a fact $q$ at an epistemic state $(M,s)$ only if its control process $P$ has this capacity and it knows $q$ at $(M,s)$ (i.e., $M,s \models \Box_A q$). The rule (IN$_{\text{e}}$) says that the agent $A$ would know the fact $q$ after $A$ receives $q$ and the epistemic state $(M,s)$ will evolve into $(M,s)\otimes (\verb"M"^{q,A},\mathbf{s})$.
The rule (COMM$_{\text{e}}$-fact-L) illustrates the epistemic interaction between two agents in an e-system, in which the agent $A$ sends the fact $q$ via the channel $\overline{a}$, while $B$ receives $q$ via the channel $a$. After such handshaking, the epistemic state $(M,s)$ evolves into $(M,s)\otimes (\verb"M"^{q,A,B},\mathbf{s})$.
%The rules (IN$_{\text{e}}$) and  (COMM$_{\text{e}}$-fact-L) refer to evolutions of epistemic states, which are depicted in terms of the operators in $\bigoplus \cup \bigotimes$.
%More explantation about SOS$_{\text{e}}$  rules may be found in~\citep{Xing2022e-calculus}.

\section{Epistemic properties}\label{sec:epistemic propoties}
This section will devote itself to giving several useful properties about epistemic interactions under AML.
%Proposition~\ref{prop:action execution preserves bisimilarity} and Proposition~\ref{prop:action execution preserves fact cognizability} reveal the following important results on action model executions:\\
%

The first two propositions reveal that action model execution preserves the indistinguishability of epistemic states and the cognition of an agent for basic facts.
%
%$\bullet\;$ \textbf{Action model execution preserves the fact cognisability}.
 \begin{proposition}[\citep{Ditmarsch2007dynamicepistemiclogic}]\label{prop:action execution preserves bisimilarity}
Let $M$ and $M'$ be two Kripke models and $\verb"M" $ be an action model. Then
%$(\verb"M",\mathbf{s})$ be a pointed action model such that $\verb"pre"^{\verb"M"}(\mathbf{u})\in \mathcal{L}_{ml}$ for each $\mathbf{u} \in\verb"W"^{\verb"M"}$. If $(\mathfrak{M},s) \underline{\leftrightarrow} (\mathfrak{M}',s') $ then
\[(M,s) \underline{\leftrightarrow} (M',s') \,\text{ implies }\,(M,s) \otimes (\verb"M",\mathbf{s}) \underline{\leftrightarrow} (M',s') \otimes (\verb"M",\mathbf{s}).\]
%\[(\mathfrak{M},s) \otimes_{\mathfrak{C }} (\verb"M",\mathbf{s}) \underline{\leftrightarrow} (\mathfrak{M}',s') \otimes_{\mathfrak{C }} (\verb"M",\langle s',\mathbf{s}).\]
\end{proposition}
%\begin{proof}
%Since $\verb"pre"^{\verb"M"}(\mathbf{s})\in \mathcal{L}_{ml}$ and $(\mathfrak{M},s) \underline{\leftrightarrow} (\mathfrak{M}',s') $, by Proposition~\ref{prop:bisi invariance}, we have
%\[\mathfrak{M},s\models_{\mathfrak{C }} \verb"pre"^{\verb"M"}(\mathbf{s})\;\;\text{ iff }\;\; \mathfrak{M}',s'\models_{\mathfrak{C }} \verb"pre"^{\verb"M"}(\mathbf{s}).\]
%Then there are two alternatives. For the first alternative, $\mathfrak{M},s\nvDash_{\mathfrak{C}} \verb"pre"^{\verb"M"}(\mathbf{s})$ and  $\mathfrak{M}',s'\nvDash_{\mathfrak{C}} \verb"pre"^{\verb"M"}(\mathbf{s})$. So $(\mathfrak{M},s) \otimes_{\mathfrak{C}} (\verb"M",\mathbf{s})=(\mathfrak{M},s)$ and  $(\mathfrak{M}',s') \otimes_{\mathfrak{C}} (\verb"M",\mathbf{s})=(\mathfrak{M}',s')$, and hence the result holds clearly. For the other one, $(\verb"M",\mathbf{s})$ is $\mathfrak{C}$-executable in both $(\mathfrak{M},s)$ and $(\mathfrak{M}',s')$, and we check $(\mathfrak{M} \otimes_{\mathfrak{C }} \verb"M",\langle s,\mathbf{s}\rangle) \underline{\leftrightarrow} (\mathfrak{M}' \otimes_{\mathfrak{C }} \verb"M",\langle s',\mathbf{s}\rangle)$.
%See~\citep[Proposition 6.21]{Ditmarsch2007dynamicepistemiclogic}.
%\end{proof}
 \begin{proposition}\label{prop:action execution preserves fact cognizability}
Let $M$  be a Kripke models and $\verb"M" $ be an action model.
%, and let an action model $(\verb"M",\mathbf{s})$ be $ \mathfrak{C}$-executable in $(\mathfrak{M},s)$.
Then for each $C\in Agent$ and $r \in Atom $,
\[M,s \models \Box_C r \;\;\text{ implies }\;\; (M,s) \otimes (\verb"M",\mathbf{s})\models \Box_C r.\]
\end{proposition}
\begin{proof}
Let $C \in Agent$, $r \in Atom$ and $M,s \models \Box_C r$. If $M,s\nvDash\verb"pre"^{\verb"M"}(\mathbf{s})$, then $(M,s) \otimes (\verb"M",\mathbf{s})=(M,s)$ and it holds trivially. Otherwise, it is enough to prove that  $M \otimes \verb"M",\langle s,\mathbf{s}\rangle \models \Box_C r$.  Assume $\langle u,\mathbf{u}\rangle \in R_C^{M \otimes\verb"M"}(\langle s,\mathbf{s}\rangle)$. Then, by Definition~\ref{def:semantics of action model logic}, we have $u \in R^{M}_C(s)$.
 From $M,s \models \Box_C r$ and $u \in R^{M}_C(s)$, it follows that $M,u \models r$, i.e., $u \in V^{M}(r)$, which implies $\langle u,\mathbf{u}\rangle \in V^{M \otimes \verb"M"}(r)$ by Definition~\ref{def:semantics of action model logic}. Hence $M \otimes \verb"M",\langle s,\mathbf{s}\rangle \models \Box_C r$.
\end{proof}
%Proposition~\ref{prop:action execution preserves fact cognizability} says that after executing an action model, an agent still cognises the basic \emph{facts} which have been cognised by him before executing this action.
 \begin{proposition}\label{prop:precondition of composition of action models}
Let $\verb"M" $ be an action model. Then, for any propositional formula $\theta$,
 \[\models \verb"pre"^{\verb"M"}(\mathbf{v}) \wedge [\,\verb"M",\mathbf{v}\,]\,\theta \longleftrightarrow \verb"pre"^{\verb"M"}(\mathbf{v}) \wedge \theta.\]
\end{proposition}
\begin{proof}
Let $(M,s)$ be a pointed Kripke model. Then

 $\quad M,s \models \verb"pre"^{\verb"M"}(\mathbf{v}) \wedge [\,\verb"M",\mathbf{v}\,]\,\theta$

\noindent $\begin{array}{ll}
    \text{iff} & M,s \models \verb"pre"^{\verb"M"}(\mathbf{v}) \text{ and } \; M \otimes\verb"M",\langle s,\mathbf{v}\rangle \models \theta \\
    \text{iff} & M,s \models \verb"pre"^{\verb"M"}(\mathbf{v}) \text{ and } \;  M,s \models \theta \qquad\qquad\qquad \qquad\quad \quad \;\text{(by Definition~\ref{def:semantics of action model logic}}\\
    &\text{ }\qquad\qquad\qquad\qquad\qquad\qquad\qquad\qquad\quad\,\text{and $\theta $ is a propositional formula)}\\
%\theta\text{ is a propositional formula)}\\
     \text{iff} & M,s \models \verb"pre"^{\verb"M"}(\mathbf{v}) \wedge \theta.  \qquad\qquad\qquad\qquad\qquad\qquad\qquad\qquad\qquad \qquad \quad\;  \qedhere
\end{array}$
\end{proof}
The following two results describe that after an agent receives a basic fact, this agent will know this fact.
\begin{proposition}\label{prop:succ1}
For any pointed Kripke model $(M,s) $,
\[(M,s) \otimes(\verb"M"^{q,B}, \mathbf{s})  \models \Box_B q,\]
where $(\verb"M"^{q,B}, \mathbf{s})$ is the pointed action model for the agent $B$ receiving the fact $q$.
\end{proposition}
\begin{proof}
Let $(M,s) $ be a pointed Kripke model. By Definition~\ref{def: K-plus-action model}, $M,s \models  \verb"pre"^{ \verb"M"^{q,B }} (\mathbf{s} )$ due to $\verb"pre"^{ \verb"M"^{q,B }} (\mathbf{s} )=\top$. So $(M,s) \otimes(\verb"M"^{q,B}, \mathbf{s}) =(M \otimes\verb"M"^{q,B}, \langle s, \mathbf{s} \rangle)$ by Definition~\ref{def:semantics of action model logic}. Let $\langle u,\mathbf{u} \rangle \in R_B^{M \otimes \verb"M"^{q,B }}(\langle s,\mathbf{s} \rangle)$. Then, by Definition~\ref{def:semantics of action model logic}, $u \in R_B^{M}(s)$, $\mathbf{u} \in \verb"R"_B^{\verb"M"^{q,B }}(\mathbf{s} ) $ (i.e., $\mathbf{u}=\mathbf{s_q}$) and $M,u \models  \verb"pre"^{ \verb"M"^{q,B }} (\mathbf{u} )$. Thus $M,u \models q$ due to $\verb"pre"^{ \verb"M"^{q,B }} (\mathbf{u} )=q$ by Definition~\ref{def: K-plus-action model}. That is, $ u \in V^{M }(q)$, and so it holds that $\langle u,\mathbf{u} \rangle \in V^{M \otimes \verb"M"^{q,B }}(q)$. Hence $M \otimes\verb"M"^{q,B}, \langle s, \mathbf{s} \rangle \models \Box_B q$.
\end{proof}
\begin{proposition}\label{prop:succ2}
 $\text{If } M,s \models \Box_A q \text{ then }  $
 \[(M,s) \otimes(\verb"M"^{q,A,B}, \mathbf{s} ) \models \Box_B q,\]
  where $(\verb"M"^{q,A,B}, \mathbf{s})$ is the pointed action model for passing $q$ from the agent $A$ to $B$.
\end{proposition}
\begin{proof}
Similar to Proposition~\ref{prop:succ1}.
\end{proof}
Proposition~\ref{prop:idempotent} says that the continuous repetition of epistemic interactions will not bring more epistemic effects and Proposition~\ref{lemma:commutativity} reveals that two executable epistemic interactions can be carried out in an arbitrary order.
To check these results, Proposition~\ref{prop:Kenv isomorphic} is needed to simplify their proofs.
\begin{proposition}\label{prop:Kenv isomorphic}
 Assume that the pointed  action models $(\verb"M"_1,\mathbf{s}_\mathbf{1})$ and $(\verb"M"_2,\mathbf{s}_\mathbf{2})$ are the ones  defined in either Definition~\ref{def: K-plus-action model} or Definition~\ref{def: K-star-action model} and they are  both executable in a Kripke model $(M,s)$.
%$(\mathfrak{M},s)$ be a Kripke model such that $\mathfrak{M},s \models \Box_A q$ if $\xi_i \equiv \circledast_{a,q,A,B} $ ($i=1,2$).
Then
\[\langle \langle s,\mathbf{s}_\mathbf{1}\rangle ,\mathbf{s}_\mathbf{2}\rangle \in W^{(M \otimes \verb"M"_1)\otimes \verb"M"_2}.\]
%\;\text{ and }\;\langle s,\langle \mathbf{s}_\mathbf{1},\mathbf{s}_\mathbf{2}\rangle \rangle \in W^{\mathfrak{M} \otimes(\verb"M"^{\xi_1}_{\mathcal{K}};\,\verb"M"^{\xi_2}_{\mathcal{K}})}. \]
%where $(\verb"M"^{\xi_1}_{\mathcal{K}},\mathbf{s}_\mathbf{1})$ and $(\verb"M"^{\xi_2}_{\mathcal{K}},\mathbf{s}_\mathbf{2})$ are the action models for $\xi_1$ and $\xi_2$ respectively.
\end{proposition}
\begin{proof}
For each $i\in\{1,2\}$, since $(\verb"M"_i,\mathbf{s}_\mathbf{i})$ is executable in $(M,s)$, $\,M,s \models \verb"pre"^{\verb"M"_i}(\mathbf{s}_{\mathbf{i}})$ and  $\langle s,\mathbf{s}_\mathbf{i} \rangle \in W^{M \otimes\verb"M"_i}$.

If $(\verb"M"_2,\mathbf{s}_\mathbf{2})$ is  defined by  Definition~\ref{def: K-plus-action model}, then we have that $\verb"pre"^{\verb"M"_2}(\mathbf{s}_{\mathbf{2}})=\top$. Clearly, $M \otimes\verb"M"_1, \langle s,\mathbf{s}_\mathbf{1} \rangle \models \verb"pre"^{\verb"M"_2}(\mathbf{s}_{\mathbf{2}})$. So, $\langle \langle s,\mathbf{s}_\mathbf{1}\rangle ,\mathbf{s}_\mathbf{2}\rangle \in W^{(M \otimes \verb"M"_1)\otimes \verb"M"_2}$.

If $(\verb"M"_2,\mathbf{s}_\mathbf{2})$ is the action model for passing $q_2\in Atom$ from the agent $A_2$ to $B_2$ defined by Definition~\ref{def: K-star-action model}, then $\verb"pre"^{\verb"M"_2}(\mathbf{s}_{\mathbf{2}})=\Box_{A_2} q_2$.
Further, by Proposition~\ref{prop:action execution preserves fact cognizability}, it follows from $M,s \models \Box_{A_2} q_2$ that $M \otimes\verb"M"_1, \langle s,\mathbf{s}_\mathbf{1} \rangle \models \Box_{A_2} q_2 $.
Hence $\langle \langle s,\mathbf{s}_\mathbf{1}\rangle ,\mathbf{s}_\mathbf{2}\rangle \in W^{(M \otimes \verb"M"_1)\otimes \verb"M"_2}$ holds.
%Moreover, by Proposition~\ref{prop:cognition relationship} (1),
%$\langle s,\langle \mathbf{s}_\mathbf{1},\mathbf{s}_\mathbf{2}\rangle \rangle \in W^{\mathfrak{M} \otimes(\verb"M"^{\xi_1}_{\mathcal{K}};\,\verb"M"^{\xi_2}_{\mathcal{K}})}$ holds.
\end{proof}
\begin{proposition}\label{prop:idempotent}
Let $(M,s) $ be a pointed Kripke model and $(\verb"M",\mathbf{s})$ be a pointed action model defined in either Definition~\ref{def: K-plus-action model} or Definition~\ref{def: K-star-action model}. Then
\[((M,s) \otimes (\verb"M",\mathbf{s}))\otimes (\verb"M",\mathbf{s}) \, \underline{\leftrightarrow}\, (M,s) \otimes (\verb"M",\mathbf{s}).\]
\end{proposition}
\begin{proof}
If $M,s \nvDash \verb"pre"^{ \verb"M"} (\mathbf{s} )$ then $(M,s) \otimes (\verb"M",\mathbf{s})=(M,s) $ and it holds trivially.
Otherwise,  $(M,s) \otimes (\verb"M",\mathbf{s})=(M \otimes\verb"M", \langle s, \mathbf{s} \rangle)$ and $(\verb"M", \mathbf{s})$ is  executable in $(M,s)$. Thus, $\langle s, \mathbf{s} \rangle \in W^{M \otimes \verb"M"}$ and $\langle \langle s,\mathbf{s}\rangle ,\mathbf{s}\rangle \in W^{(M \otimes \verb"M")\otimes \verb"M"}$ by Proposition~\ref{prop:Kenv isomorphic}. Further, we have
\[((M,s) \otimes (\verb"M",\mathbf{s}))\otimes (\verb"M",\mathbf{s})=((M \otimes\verb"M") \otimes\verb"M", \langle \langle s, \mathbf{s}\rangle, \mathbf{s} \rangle).\]
In the following, we intend to check that
\[\mathcal{Z}: ((M \otimes\verb"M") \otimes\verb"M", \langle \langle s, \mathbf{s}\rangle, \mathbf{s} \rangle)\, \underline{\leftrightarrow}\, (M \otimes\verb"M", \langle s, \mathbf{s} \rangle)\]
where
\[\mathcal{Z} \triangleq \left\{\langle\langle\langle u,\mathbf{u}\rangle,\mathbf{u} \rangle, \langle u,\mathbf{u}\rangle \rangle :\;\begin{subarray}{1} \langle u,\mathbf{u}\rangle \in W^{M \otimes \verb"M"}\;\; \text{ and }\\
\langle\langle u,\mathbf{u}\rangle,\mathbf{u} \rangle \in W^{(M \otimes \verb"M")\otimes \verb"M"} \end{subarray}\; \right\}.\]

Let $\langle\langle u,\mathbf{u}\rangle,\mathbf{u} \rangle \mathcal{Z}\langle u,\mathbf{u} \rangle  $. Then $ \langle\langle u,\mathbf{u}\rangle,\mathbf{u} \rangle \in W^{(M \otimes \verb"M")\otimes \verb"M"} $ and $\langle u,\mathbf{u}\rangle\in W^{M \otimes \verb"M"}$.
By Definition~\ref{def:semantics of action model logic}, the conditions \textbf{(atoms)} and \textbf{(forth)} hold trivially. We below consider \textbf{(back)}. Let $\langle u,\mathbf{u} \rangle R_C^{M\otimes \verb"M"} \langle u',\mathbf{u}' \rangle$ and $C\in Agent$. It suffices to show that $\langle \langle u',\mathbf{u}'\rangle ,\mathbf{u}'\rangle \in W^{(M \otimes \verb"M")\otimes \verb"M"}$. Because of $\langle u',\mathbf{u}'\rangle \in W^{M \otimes \verb"M"}$, we have $M,u' \models \verb"pre"^{ \verb"M"} (\mathbf{u}' )$.
Since there exists no arrow entering $\mathbf{s}$  by Definition~\ref{def: K-plus-action model}  and Definition~\ref{def: K-star-action model}, due to $\langle u, \mathbf{u}\rangle R_C^{M \otimes \verb"M"}  \langle u', \mathbf{u}'\rangle$, we get $\mathbf{u}'\neq \mathbf{s}$. So
\begin{equation}
\verb"pre"^{ \verb"M"}( \mathbf{u}' )\in \{\top\} \cup Atom. \; \label{eq:environmentexample-4}
\end{equation}
Thus, by Proposition~\ref{prop:precondition of composition of action models}, from $M,u' \models \verb"pre"^{ \verb"M"} (\mathbf{u}' ) \wedge \verb"pre"^{ \verb"M"} (\mathbf{u}' )$ and (\ref{eq:environmentexample-4}), it follows  that $M,u' \models   [\,\verb"M",\mathbf{u}'\,]\, \verb"pre"^{ \verb"M"}(\mathbf{u}')$.
Therefore $\langle \langle u',\mathbf{u}'\rangle, \mathbf{u}'\rangle \in W^{(M \otimes \verb"M")\otimes \verb"M"}$ holds, as desired.
\end{proof}
\begin{proposition}\label{lemma:commutativity}
Let $(M,s) $ be a pointed Kripke model. Assume that $(\verb"M",\mathbf{s})$ and $ (\verb"N", \mathbf{t})$ be two pointed action models defined in either Definition~\ref{def: K-plus-action model} or Definition~\ref{def: K-star-action model}  and  both executable in  $(M,s)$. Then
\[((M,s) \otimes (\verb"M",\mathbf{s}))\otimes (\verb"N",\mathbf{t})\, \underline{\leftrightarrow} \,((M,s) \otimes (\verb"N",\mathbf{t}))\otimes (\verb"M",\mathbf{s}).\]
\end{proposition}
\begin{proof}
Since both $(\verb"M", \mathbf{s})$ and $ (\verb"N", \mathbf{t})$ are executable in $(M,s)$,
 by Proposition~\ref{prop:Kenv isomorphic}, we have that $\langle \langle s,\mathbf{s}\rangle ,\mathbf{t}\rangle \in W^{(M \otimes \verb"M")\otimes \verb"N"}$,  $\langle \langle s,\mathbf{t}\rangle ,\mathbf{s}\rangle \in W^{(M \otimes \verb"N")\otimes \verb"M"}$ and
\[\begin{array}{lll}
((M,s) \otimes (\verb"M",\mathbf{s})) \otimes (\verb"N",\mathbf{t}) &=&((M \otimes\verb"M") \otimes\verb"N", \langle \langle s, \mathbf{s}\rangle, \mathbf{t} \rangle) \\
((M,s) \otimes (\verb"N",\mathbf{t})) \otimes (\verb"M",\mathbf{s})&=&((M \otimes\verb"N") \otimes\verb"M", \langle \langle s, \mathbf{t}\rangle, \mathbf{s} \rangle).
\end{array}\]
In the following, we intend to check
\[\mathcal{Z}: ((M \otimes\verb"M") \otimes\verb"N", \langle \langle s, \mathbf{s}\rangle, \mathbf{t} \rangle)\, \underline{\leftrightarrow} \,((M \otimes\verb"N") \otimes\verb"M", \langle \langle s, \mathbf{t}\rangle, \mathbf{s} \rangle)\]
where
\[\mathcal{Z} \triangleq \left\{\langle\langle\langle u, \mathbf{u}\rangle,\mathbf{v} \rangle,\, \langle\langle u,\mathbf{v}\rangle,\mathbf{u} \rangle\rangle :  \begin{subarray}{1} \langle\langle u, \mathbf{u}\rangle,\mathbf{v} \rangle\in W^{(M \otimes \verb"M") \otimes \verb"N"}\text{ and } \\
\langle\langle u,\mathbf{v}\rangle,\mathbf{u} \rangle\in W^{(M \otimes \verb"N")\otimes \verb"M"} \end{subarray}\right\}.\]

 Suppose   $ \langle\langle u, \mathbf{u}\rangle,\mathbf{v} \rangle \mathcal{Z }\langle \langle u,\mathbf{v}\rangle,\mathbf{u} \rangle$.
Then, we have  that $\langle \langle u,\mathbf{u}\rangle,\mathbf{v} \rangle \in W^{(M \otimes \verb"M")\otimes \verb"N"}$
%, $u \in W^{\mathfrak{M}}$,  $\mathbf{u} \in \verb"W"^{ \verb"M"^{\xi_1}_{\mathcal{K}} }$, $\mathbf{v} \in \verb"W"^{ \verb"M"^{\xi_2}_{\mathcal{K}} }$
 and $\langle \langle u,\mathbf{v}\rangle,\mathbf{u} \rangle \in W^{(M \otimes \verb"N")\otimes \verb"M"}$.
By Definition~\ref{def:semantics of action model logic},  the condition \textbf{(atoms)} holds trivially. We will below check \textbf{(forth)}, and \textbf{(back)} may be checked similarly and omitted.

Let $\langle\langle u, \mathbf{u}\rangle,\mathbf{v} \rangle R_C^{(M \otimes \verb"M")\otimes \verb"N"}  \langle\langle u', \mathbf{u}'\rangle,\mathbf{v}' \rangle$ and $C\in Agent$. It suffices to show  $\langle \langle u',\mathbf{v}'\rangle,\mathbf{u}' \rangle \in W^{(M \otimes \verb"N")\otimes \verb"M"}$. Clearly, due to $\langle \langle u',\mathbf{u}'\rangle,\mathbf{v}' \rangle \in W^{(M \otimes \verb"M")\otimes \verb"N"}$, it holds that $ \langle u',\mathbf{u}'\rangle \in W^{M \otimes \verb"M"}$ and $M \otimes \verb"M",\langle u',\mathbf{u}'\rangle\models \verb"pre"^{ \verb"N"}(\mathbf{v}')$, and hence
\begin{equation}
M,u' \models  \verb"pre"^{ \verb"M"}(\mathbf{u}') \wedge [\,\verb"M",\mathbf{u}'\,]\, \verb"pre"^{ \verb"N"}(\mathbf{v}'). \; \label{eq:environmentexample-3}
\end{equation}
Since there exists no arrow entering $\mathbf{s}$ ($\mathbf{t}$) by Definition~\ref{def: K-plus-action model}  and Definition~\ref{def: K-star-action model}, due to $\langle\langle u, \mathbf{u}\rangle,\mathbf{v} \rangle R_C^{(M \otimes \verb"M")\otimes \verb"N"}  \langle\langle u', \mathbf{u}'\rangle,\mathbf{v}' \rangle$, we get $\mathbf{u}'\neq \mathbf{s}$ and $\mathbf{v}'\neq \mathbf{t}$. So
\[\verb"pre"^{ \verb"M"}( \mathbf{u}' ), \verb"pre"^{ \verb"N"}( \mathbf{v}' )\in \{\top\} \cup Atom .\]
Then, by Proposition~\ref{prop:precondition of composition of action models}, it follows from (\ref{eq:environmentexample-3}) that
\[M,u' \models  \verb"pre"^{ \verb"M"}(\mathbf{u}') \wedge \verb"pre"^{ \verb"N"}(\mathbf{v}').\]
Thus, by applying Proposition~\ref{prop:precondition of composition of action models} again, it holds that
\[M,u' \models  \verb"pre"^{ \verb"N"}(\mathbf{v}') \wedge [\,\verb"N",\mathbf{v}'\,]\, \verb"pre"^{ \verb"M"}(\mathbf{u}').  \]
Therefore $\langle \langle u',\mathbf{v}'\rangle, \mathbf{u}'\rangle \in W^{(M \otimes \verb"N")\otimes \verb"M"}$ holds, as desired.
\end{proof}
\section{Synchronous communication}\label{sec:example-synchronous communication}
Point-to-point communications between agents are fundamental and often used in many applications. In the e-calculus, the SOS rule (COMM$_{\text{e}}$-fact-L) is adopted to model point-to-point communication pattern. This rule describes that, for two agents $A$ and $B$ in an e-system, a communication of the fact $q$ between them will bring their common knowledge about $q$. In such situations, facts are \emph{immediately} received at the same time they are sent. It is a natural problem  which agents will participate in a communication. In the e-calculus, with the help of the operators, the agents participating in a communication are either \emph{determined} or \emph{non-determined}.

For the former, as an example, consider the e-system  below
 \[E \triangleq (\nu b,c)\,([\,\overline{b} p.\, \overline{c} p.\,\mathbf{0} \,]_A \parallel [\,b(\chi).\,\mathbf{0}\,]_B\parallel [\,c(\chi').\,\mathbf{0}\,]_C). \]
Let $(M,s)$ be a pointed Kripke model with $M,s\models \Box_A p$. Intuitively, $E$ at $(M,s)$ has the evolution path as follows
%\[\begin{array}{lll}
\[\begin{array}{ll}
(M,s)\blacktriangleright E  \xlongrightarrow  {\langle b,p,A,B\rangle}(M,s)\otimes (\verb"M"^{p,A,B},\mathbf{s}) \blacktriangleright E'\\
\qquad  \xlongrightarrow  {\langle c,q,A,C\rangle} ((M,s)\otimes (\verb"M"^{p,A,B},\mathbf{s}))\otimes  (\verb"N"^{q,A,C},\mathbf{t}) \blacktriangleright E''.
 \end{array}\]
%\end{array}\]
with
\[\begin{array}{lll}
E' &\triangleq & (\nu b,c)\,([\, \overline{c} q.\,\mathbf{0} \,]_A \parallel [\,\mathbf{0}\,]_B\parallel [\,c(\chi').\,\mathbf{0}\,]_C)\\
E'' & \triangleq & (\nu b,c)\,([\,\mathbf{0} \,]_A \parallel [\,\mathbf{0}\,]_B\parallel [\,\mathbf{0}\,]_C)
\end{array}\]
Due to $M,s\models \Box_A p$, Proposition~\ref{prop:action execution preserves fact cognizability} and the SOS rule (COMM$_{\text{e}}$-fact-L), the fact $p$ will be passed from the agent $A$ to $B$ and after this, to $C$.
%\end{example}

For the latter, we give an example about the cooperation of cognitive robots in intelligent machine communities.
\begin{example}\label{example2-ptop}
Consider the application scenario where the cognitive robot $A$ and anyone of the cognitive robots $D_k $ ($1\leq k\leq 3$)  (all of these robots  have jacks) work together to move a table. The robot $A$ has been in place and the table may be moved only when another robot is also in place.  We make the following assumptions.

$(1) $ A robot will know that he has been in place when he reaches
%\begin{adjustwidth}{1.5em}{0em}
 the predetermined position. The propositional letter $q_k$ is used to denote the basic fact that the robot $D_k$  is in place.
% \end{adjustwidth}

$(2) $ After the robot $D_k$ passes the fact $q_k$ to the robot $A$, the
%\begin{adjustwidth}{1.5em}{0em}
fact $q_k$ will become the common knowledge between these two robots.
% \end{adjustwidth}

$(3)$ The robots $A$ and $D_k$ will start their jacks to move the
%\begin{adjustwidth}{1.5em}{0em}
table when the robots $A$ knows that the robot $D_k$ is in place.
% \end{adjustwidth}

 This scenario may be specified as
 \[\nu a\,(\,[\, a (\chi).\,X \,]_A \parallel \underset{1\leq k\leq 3}{ \parallel}[\,\overline{a_k} \, q_k.\,\mathbf{0}\,]_{D_k}).\]
 Due to the assumptions  $(1)$ and $(2)$ and the SOS rule (COMM$_{\text{e}}$-fact-L),   whenever the robot $D_k$ reaches its predetermined position for some $1\leq k\leq 3$, the fact $q_k$ will be passed to the robot $A$ by the robot $D_k$ and hence it will become a common knowledge between them. Therefore, the jacks of the robots $A$ and $D_k$ will be triggered and the table will be moved due to the assumption $(3)$.
\end{example}
%\begin{remark}\label{rem:ptop}
%In Example~\ref{example1-ptop}, the epistemic communication is \emph{deterministic} in terms of the agents participating in this communication, while the one in Example~\ref{example2-ptop} is non-deterministic.
Moreover, if the scenario in Example~\ref{example2-ptop}  is formalized by
  \[\nu a\,([\, P \,]_A \parallel \underset{1\leq k\leq 3}{ \parallel}[\,\overline{a_k} q_k.\,\mathbf{0}\,]_{D_k})\]
with
\[P \triangleq  a_1 (\chi).\,X_1+a_2 (\chi).\,X_2+a_3 (\chi).\,X_3, \]
then by the SOS rule (SUM), this e-system will be only able to exercise the capabilities of $X_k$ whenever it is the robots $A$ and $D_k$ to cooperate to move the table for some $1\leq k\leq 3$.
%\end{remark}
\section{Asynchronous communication}\label{sec:example-asynchronous communication}
This section  will apply the e-calculus to formalize asynchronous communications, in which there may be an unpredictable \emph{delay} between sending and receiving messages. Then, it is a natural problem in what order the agents will receive the messages which have been sent and are waiting to be received. Due to the rules (OPEN) and (CLOSE),  the e-calculus can create and transmit \emph{new} bound channel names. Based on this, to capture this \emph{asynchrony}, we intend to construct a \emph{buffer pool} with storage cells addressed by channel names,  which is used to store the announced basic facts (propositional letters) and from which the receiver can read these basic facts when these agents are ready to process them. The e-calculus can construct such buffer pools visited in different orders. In the following, we intend to give two examples: one is in the order in which the announced facts are sent (First-in-first-out, FIFO), and the other is in an arbitrary order.
\begin{notation}\label{not-notation1}
To ease the notations, we shall adopt the following abbreviations
\[\begin{array}{lll}
\underset{1\leq i\leq n}{ \parallel}  H_i &\triangleq &H_1 \parallel \cdots \parallel H_n\\
\underset{1\leq i\leq n}{ \mid}  X_i &\triangleq &X_1 \mid  \cdots \mid X_n\\
%\underset{1\leq i\leq n}{ \LARGE{\bullet}}  \pi_i& \triangleq & \pi_1.\,  \cdots.\, \pi_n\\
%\;\;\;\;\;(\pi)^n &\triangleq &\underset{ n}{\underbrace{\pi.\,\cdots.\,\pi}}\\
(\nu z_1, \cdots, z_n)&\triangleq &\nu z_1 \cdots \nu z_n.
\end{array}\]
\end{notation}
\subsection{Reading in FIFO order}\label{subsec-asynchronous communication-FIFO}
In~\citep[Subsection 1.2.3]{Sangiorgi2001pi-calculus}, the $\pi$-calculus builds unbounded chains of storage cells and other processes can navigate the chains by retrieving and following the links stored in the cells that comprise them. Similarly, this subsection will define an e-system as an example which constructs a unbounded \emph{chained} buffer pool, whose storage cells are used to store the basic \emph{facts} obtained from some external source and the \emph{links}. By retrieving and following these links, the agents can navigate the buffer pool and read these facts from the storage cells in the order FIFO.

At first glance, by the similar strategy as in~\citep[Subsection 1.2.3]{Sangiorgi2001pi-calculus}, we may consider the following e-system $G$ where
\[\begin{array}{lll}
G& \triangleq &(\nu h,a,c)\,([P\mid Q\mid Z]_{D^{\blacklozenge}} \parallel [X ]_{D}\parallel [Y]_{A} \parallel [Y]_{B} )\\
P &\triangleq & h(z).\,a(\chi).\,!\,\overline{h} z.\, \overline{z} \chi.\,\mathbf{0}\\
Q&\triangleq & !\,\nu b\, (\overline{c} b.\,b(y).\,a(\chi_1) .\, !\,\overline{b} y.\,\overline{y} \chi_1.\,\mathbf{0})\\
Z& \triangleq & \nu e\, (\overline{e} h.\,\mathbf{0}\mid \,!\, e(z_1).\,c(y_1).\,\overline{z_1}\, y_1.\, \overline{e} y_1.\,\mathbf{0} )\\
Y &\triangleq & \nu g\, (\overline{g} h.\,\mathbf{0}\mid \,!\, g(z_2).\,z_2(y_2).\, y_2(\chi_2).\, \overline{g} y_2.\,\mathbf{0} )\\
X & \triangleq & (\nu a_0,a_1)\,(\overline{a_0} a_1.\,\mathbf{0}\mid \,!\,a_0(x).\,a^{\ast}(\chi^{\ast}) .\,\overline{a} \chi^{\ast}.\,\overline{a_0} a_1.\,\mathbf{0}).
%X & \triangleq & \underset{1\leq i \leq n}{ \LARGE{\bullet}} \overline{a} q_i.\,\mathbf{0}.
\end{array}\]
%where $a, b,c,g: NPri$.
The process $P$ can be thought of as a \emph{heading cell} named $h$, which points to the first cell storing facts.
%: on receiving a name via $h$ and receiving a fact via $a$, it repeatedly sends this name via $h$ and sends this fact via this name.
The process $Q$ is a generator of storage cells.
%: the combination $!\,\nu b\, \overline{c} b$ represents that it can repeatedly send fresh names via $c$; on sending a name, it continues as the composition of a cell bearing that name and the generator.
The agent $D^{\blacklozenge} $ is responsible for managing the chained buffer pool builded through $Z$. The agent $D$ obtains basic facts from an external source via $a^{\ast}$ and stores these facts into the buffer pool in the order that they are obtained. The agents $A$ and $B$ read these basic facts from the buffer pool in the chain order. See~\citep{Sangiorgi2001pi-calculus} for more ideas behind this construction.

To see how these components of $G$ realizes these goals, we below give one possible evolution of $ G$ at a given initial epistemic state $(M,s)$ in detail.

The evolution starts with an interaction between the components of $Z$ via the bound name $e$:
\[ (M,s)\blacktriangleright G  \xlongrightarrow {\tau} (M,s) \blacktriangleright G_1\]
where
\begin{align*}
G_1\; &\triangleq \;  (\nu h,a,c)\,([P\mid Q\mid Z_1]_{D^{\blacklozenge}} \parallel [X ]_{D}\parallel [Y]_{A} \parallel [Y]_{B} )\\
Z_1 \;&\triangleq \; \nu e (\mathbf{0} \mid c(y_1).\overline{h} y_1. \overline{e} y_1.\mathbf{0} \mid \,!\, e(z_1).\,c(y_1).\overline{z_1}  y_1. \overline{e} y_1.\mathbf{0} ).
\end{align*}
The evolution continues between $Q$ and $Z_1$:
\begin{equation}
(M,s)\blacktriangleright G_1  \xlongrightarrow {\tau} (M,s) \blacktriangleright G_2\; \label{eq:example-asychronous-2}
\end{equation}
where
\[\begin{array}{lll}
%\begin{align*}
G_2 &\triangleq &(\nu h,a,c,b_1)\,([P\mid Q_1\mid Q\mid Z_2]_{D^{\blacklozenge}} \parallel [X ]_{D}\parallel H )\\
Z_2& \triangleq & \nu e\, (\mathbf{0} \mid\overline{h} b_1.\, \overline{e} b_1.\,\mathbf{0} \mid \,!\, e(z_1).\,c(y_1).\,\overline{z_1}\, y_1.\, \overline{e} y_1.\,\mathbf{0} )\\
Q_1 &\triangleq & b_1(y).\,a(\chi_1) .\, !\,\overline{b_1} y.\,\overline{y} \chi_1.\,\mathbf{0}.\\
H & \triangleq &  [Y]_{A} \parallel [Y]_{B}
\end{array}\]
%\end{align*}
This expresses that a new cell $Q_1$, named $b_1$, is created and its name is transmitted to $Z_1$ via $c$.
Next, the evolution continues with an interaction between $Z_2$ and $P$ via $h$, which expresses that $Z_2$ stores $b_1$ into the heading cell $h$:
\begin{equation}
 (M,s)\blacktriangleright G_2 \xlongrightarrow {\tau} (M,s) \blacktriangleright G_3\; \label{eq:example-asychronous-1}
\end{equation}
where
\[\begin{array}{lll}
%\begin{align*}
G_3 &\triangleq &(\nu h,a,c,b_1)\,([P_1\mid Q_1\mid Q\mid Z_3]_{D^{\blacklozenge}} \parallel [X ]_{D}\parallel H )\\
P_1 &\triangleq & a(\chi).\,!\,\overline{h} b_1.\, \overline{b_1} \chi.\,\mathbf{0}\\
Z_3 &\triangleq & \nu e\, (\mathbf{0} \mid\overline{e} b_1.\,\mathbf{0} \mid \,!\, e(z_1).\,c(y_1).\,\overline{z_1}\, y_1.\, \overline{e} y_1.\,\mathbf{0} ).
\end{array}\]
%\end{align*}
In the next three evolution steps, after an interaction between the components of $X$ via the bound name $a_0 $, the agent $D$ receives a basic fact $q_1$ via $a^{\ast}$ from the external source after which $D$ knows $q_1$ (i.e., $(M,s)\otimes  (\verb"M"^{q_1,D},\mathbf{s}) \models \Box_{D} q_1 $) by Proposition~\ref{prop:succ1} so that $D$ can send $q_1$  due to the rule (OUT$_{\text{e}}$), then by communicating $q_1$ between $P_1$ and the continuation of $X$ via $a$, the fact $q_1$ is stored into the cell $b_1$:
%\[ M\blacktriangleright G_3 \xlongrightarrow {\tau} M_1 \blacktriangleright G_4\]
\[\begin{array}{lll}
 (M,s)\blacktriangleright G_3& \xlongrightarrow {\tau} (M,s) \blacktriangleright G_4\\
  &\xlongrightarrow {\langle a^{\ast} q_1,D\rangle} (M,s)\otimes (\verb"M"^{q_1,D},\mathbf{s_1}) \blacktriangleright G_5 \\
  &\xlongrightarrow {\langle a,q_1,D,D^{\blacklozenge}\rangle} (M_1,s_1) \blacktriangleright G_6
 \end{array}\]
where
\[(M_1,s_1)  \triangleq  ((M,s)\otimes (\verb"M"^{q_1,D},\mathbf{s_1}))\otimes  (\verb"M"^{q_1,D,D^{\blacklozenge}},\mathbf{t_1})\]	
and
\begin{align*}
G_4 &\triangleq (\nu h,a,c,b_1)\,([\,P_1\mid Q_1\mid Q\mid Z_3\,]_{D^{\blacklozenge}} \parallel [\,X_1 \,]_{D}\parallel H )\\
G_5 &\triangleq (\nu h,a,c,b_1)\,([\,P_1\mid Q_1\mid Q\mid Z_3\,]_{D^{\blacklozenge}} \parallel [\,X_2 \,]_{D}\parallel H )\\
G_6 &\triangleq (\nu h,a,c,b_1)\,([\,P_2\mid Q_1\mid Q\mid Z_3\,]_{D^{\blacklozenge}} \parallel [\,X_3 \,]_{D}\parallel H )\\
X_1 & \triangleq  (\nu a_0,a_1)\,(\mathbf{0}\mid a^{\ast}(\chi^{\ast}) .\,\overline{a}\, \chi^{\ast}.\,\overline{a_0}\, a_1.\,\mathbf{0}\mid \,!\,a_0(x).\,a^{\ast}(\chi^{\ast}) .\,\overline{a} \chi^{\ast}.\,\overline{a_0} a_1.\,\mathbf{0})\\
X_2 & \triangleq (\nu a_0,a_1)(\mathbf{0}\mid\overline{a} q_1.\overline{a_0} a_1.\mathbf{0}\mid  !a_0(x).a^{\ast}(\chi^{\ast}) .\overline{a} \chi^{\ast}.\overline{a_0} a_1.\mathbf{0})\\
X_3 & \triangleq (\nu a_0,a_1)\,(\mathbf{0}\mid \overline{a_0}\, a_1.\,\mathbf{0}\mid \,!\,a_0(x).\,a^{\ast}(\chi^{\ast}) .\,\overline{a} \chi^{\ast}.\,\overline{a_0}\, a_1.\,\mathbf{0})\\
P_2 &\triangleq \,!\,\overline{h} b_1.\, \overline{b_1}\, q_1.\,\mathbf{0}
%X_1 &\triangleq & \underset{2\leq i \leq n}{ \LARGE{\bullet}} \overline{a} q_i.\,\mathbf{0}.
\end{align*}
Here, $(\verb"M"^{q_i,D},\mathbf{s_i})$ is the pointed action model for the agent $D$ receiving $q_i$ and $(\verb"M"^{q_i,D,D^{\blacklozenge}},\mathbf{t_i})$  for passing $q_i$ from $D$ to $D^{\blacklozenge}$ with $1\leq i \leq n $. Now, through $P_2$, the agents can read $q_1$ from the cell $b_1$, since $(M_1,s_1)\models \Box_{D^{\blacklozenge}} q_1 $ by Proposition~\ref{prop:succ2}.

We easily see that $X_3$ is the same as $X$, except that they are connected to the different cells. Hence the construction of the buffer pool will continue. The next six steps are similar to the first six steps, which will build a second new cell, and store its name into the cell $b_1$ and a second received fact $q_2$ into this new cell. After repeating this $n$ ($\geq 1$) times, we construct a chained $n$-cell buffer pool storing the basic facts  $q_1$, $\cdots$, $q_n$ received from the external source, with the current epistemic state as follows
\[
(M',s')\triangleq (M,s)\otimes (\verb"M"^{q_1,D},\mathbf{s_1})\otimes  (\verb"M"^{q_1,D,D^{\blacklozenge}},\mathbf{t_1})\cdots\otimes (\verb"M"^{q_n,D},\mathbf{s_n})\otimes  (\verb"M"^{q_n,D,D^{\blacklozenge}},\mathbf{t_n}).
\]
%\[M' \triangleq \circledast_{q_n,D,D^{\blacklozenge}}\oplus_{q_n,D}\cdots\circledast_{q_1,D,D^{\blacklozenge}} \oplus_{ q_1,D}(M).\]
So, by Proposition~\ref{prop:succ1}, Proposition~\ref{prop:succ2} and Proposition~\ref{prop:action execution preserves fact cognizability},
 \begin{equation}
(M',s') \models \Box_{D^{\blacklozenge}} q_i \quad \text{ for each } 1\leq i \leq n.\; \label{eq:as-1}
\end{equation}
Thus, the agents $A$ and $B$ can read these facts from the buffer pool. The agent $A$ reads $q_1$ by, after an interaction between the components of $Y$ via $g$, communicating $b_1$ between $P_2$ and the continuation of $Y$ via $h$ and then $q_1$ via $b_1$, which leads to the following state
\[\begin{array}{ll}
(M',s')\otimes (\verb"M"^{q_1,D^{\blacklozenge},A},\mathbf{u_1 })\blacktriangleright(\nu h,a,c,b_1)\,( \\
\qquad [\,\mathbf{0}\mid P_2\mid Q_1\mid Q\mid Z_3\,]_{D^{\blacklozenge}}\parallel [\,X_3 \,]_{D}\parallel [\,Y_1\,]_{A} \parallel [\,Y\,]_{B}\, )
\end{array}\]
where
\[Y_1\;\; \triangleq \;\; \nu g\, (\mathbf{0}\mid\overline{g} b_1.\,\mathbf{0}\mid \,!\, g(z_2).\,z_2(y_2).\, y_2(\chi_2).\, \overline{g} y_2.\,\mathbf{0} ). \]
Here, $(\verb"M"^{q_i,D^{\blacklozenge},A},\mathbf{u_i })$ is the pointed action model for passing $q_i$ from $D^{\blacklozenge}$ to $A$ with $1\leq i \leq n $.

However, there are some \emph{unexpected} possibilities (i.e., flaws) through the evolution of $(M,s) \blacktriangleright G$. We have to revise these flaws below.\\

\noindent \textbf{(Flaw 1)} ~~In the evolution step (\ref{eq:example-asychronous-1}), it is not difficult to see that, by activating another copy of the replication in $Z_3$, the evolution is indeed able to continue with carrying out another construction step. This means that,
the facts might \emph{not} be stored in \emph{consecutive} cells of the buffer pool.\\

\noindent \textbf{(Revise 1)} ~~At first sight, we may use the name $d$ to denote that a fact has been stored into a new cell. Motivated by this idea, $P$, $Q$ and $Z$ are redefined as follows, with the added parts underlined.
\[\begin{array}{lll}
G& \triangleq &(\nu h,a,c,\underline{f,d})\,([\,P\mid Q\mid Z\,]_{D^{\blacklozenge}} \parallel [\,X \,]_{D}\parallel H )\\
P &\triangleq & h(z).\,a(\chi).\,\underline{\overline{f} d.}\, !\,\overline{h} z.\, \overline{z} \chi.\,\mathbf{0}\\
Q&\triangleq & !\,\nu b\, (\overline{c} b.\,b(y).\,a(\chi_1).\,\underline{\overline{f} d} .\, !\,\overline{b} y.\,\overline{y} \chi_1.\,\mathbf{0})\\
Z& \triangleq & \nu e\, (\overline{e} h.\,\mathbf{0}\mid \,!\, e(z_1).\,c(y_1).\,\overline{z_1} y_1.\, \underline{f(x).}\, \overline{e} y_1.\,\mathbf{0} )\\
Y &\triangleq & \nu g\, (\overline{g} h.\,\mathbf{0}\mid \,!\, g(z_2).\,z_2(y_2).\, y_2(\chi_2).\, \overline{g} y_2.\,\mathbf{0} )\\
X & \triangleq & (\nu a_0,a_1)\,(\overline{a_0}\, a_1.\mathbf{0}\mid \,!a_0(x_1).a^{\ast}(\chi^{\ast}) .\overline{a} \chi^{\ast}.\overline{a_0}\, a_1.\mathbf{0}).
%X & \triangleq & \underset{1\leq i \leq n}{ \LARGE{\bullet}} \overline{a} q_i.\,\mathbf{0}\text{¡£}
\end{array}\]

But, such simple revisions will cause the additional defects as follows.\\
%$\text{ }$

\noindent\textbf{(Flaw 2)} ~~After the evolution step (\ref{eq:example-asychronous-2}), the evolution of $(M,s) \blacktriangleright G_2$ may also continue with an interaction between the components of $Y$ via the bound name $g$, and further, it is possible that $Z_2$ sends $b_1$ to the continuation of $Y$, not to $P$, via $h$. But this is not desired. Instead of this, we want the agents $A$ and $B$ to read from the cells, e.g., through $P_2$.\\

\noindent \textbf{(Revise 2)} ~~We alias each cell, such as $h$ with the alias $h'$ and $b_1$ with the alias $b'_1$, and each cell is read by its \emph{alias}. Correspondingly, $P$, $Q$, $Z$ and $Y$ are revised as follows, with the revised parts underlined.
\[\begin{array}{lll}
G& \triangleq &(\nu h,a,c,\underline{h',c'},f,d)\,([P\mid Q\mid Z]_{D^{\blacklozenge}} \parallel [X ]_{D}\parallel H )\\
P &\triangleq & h(z).\,a(\chi).\,\overline{f} d.\, !\,\underline{\overline{h'} z}.\, \overline{z}\chi .\,\mathbf{0}\\
Q&\triangleq & !\,(\nu b,b')\, (\overline{c} b.\,\underline{\overline{c'} b'}.\,b(y).\,a(\chi_1).\,\overline{f} d .\, !\,\underline{\overline{b'} y}.\,\overline{y}\chi_1.\,\mathbf{0})\\
Z& \triangleq & \nu e\, (\overline{e} h.\mathbf{0}\mid \,!\, e(z_1).\,c(y_1).\,\underline{c'(y'_1)}.\,\underline{\overline{z_1} y'_1}.\, f(x).\, \overline{e} y_1.\mathbf{0} )\\
Y &\triangleq & \nu g\, (\underline{\overline{g} h'}.\,\mathbf{0}\mid \,!\, g(z_2).\,z_2(y_2).\, y_2(\chi_2).\, \overline{g} y_2.\,\mathbf{0} )\\
X & \triangleq & (\nu a_0,a_1)\,(\overline{a_0}\, a_1.\mathbf{0}\mid \,!\,a_0(x_1).a^{\ast}(\chi^{\ast}) .\overline{a} \chi^{\ast}.\overline{a_0}\, a_1.\mathbf{0}).
%X & \triangleq & \underset{1\leq i \leq n}{ \LARGE{\bullet}} \overline{a} q_i.\,\mathbf{0}.
\end{array}\]
The process $Q$ builds a new cell by sending a fresh bound name via $c$ and \emph{aliases} this cell by sending a fresh bound name via $c'$.\\

Similarly, we may give the evolution of the final e-system $G$ revised at the epistemic state $(M,s) $. Through this evolution, we easily see that it is the \emph{aliases} that are stored in the cells as links, and specially, $P_2$ and $Q_1$ should be redefined as
\[\begin{array}{lll}
P_2 &\triangleq &!\,\underline{\overline{h'} b'_1}.\, \underline{\overline{b'_1} q_1}.\,\mathbf{0}\\
Q_1 &\triangleq & b_1(y).\,a(\chi_1) .\, \overline{f} d.\,!\,\underline{\overline{b'_1} y}.\,\overline{y} \chi_1.\,\mathbf{0}.
\end{array}\]

Further, after $n$ cells have been constructed, which store $q_1$, $\cdots$, $q_n$ respectively, and then these facts have just been read by both $A$ and $B$, all the evolution paths may reach the same \emph{state}, with the same epistemic state up to $\underline{\leftrightarrow}$ by (\ref{eq:as-1}), Proposition~\ref{prop:action execution preserves fact cognizability} and Proposition~\ref{lemma:commutativity} as follows
\[\begin{array}{lll}
(M',s')\otimes (\verb"M"^{q_1,D^{\blacklozenge},A},\mathbf{u_1 })\cdots\otimes (\verb"M"^{q_n,D^{\blacklozenge},A},\mathbf{u_n })\qquad\qquad\\
\qquad\quad\;\qquad\qquad\otimes (\verb"M"^{q_1,D^{\blacklozenge},B},\mathbf{v_1 })\cdots\otimes (\verb"M"^{q_n,D^{\blacklozenge},B},\mathbf{v_n })
\end{array}\]
%\[\circledast_{q_n,D^{\blacklozenge},B}\cdots\circledast_{q_1,D^{\blacklozenge},B} \circledast_{q_n,D^{\blacklozenge},A}\cdots\circledast_{q_1,D^{\blacklozenge},A}(M'), \]
and the same e-system up to isomorphism of LTSs as follows
%\begin{multline*}
\begin{multline*}
(\nu h,a,c,b_1,\cdots,b_n,h',c',b'_1,\cdots,b'_n,f,d)\,(
[P']_{D^{\blacklozenge}} \parallel [X ]_{D}\parallel [Y']_{A} \parallel [Y']_{B} \,)
\end{multline*}
%\end{multline*}
where
\[\begin{array}{lll}
P' &\triangleq & P_2\mid Q' \mid Q_n\mid Q\mid Z'\\
Q' &\triangleq&  \underset{1\leq i \leq n-1}{ \mid} \,!\, \overline{b'_i} b'_{i+1}.\, \overline{b'_{i+1}} q_{i+1}.\,\mathbf{0}   \\
Q_n &\triangleq & b_n(y).\,a(\chi_1).\,\overline{f} d.\, !\,\overline{b'_n} y.\, \overline{y} \chi_1.\,\mathbf{0}\\
Z'& \triangleq & \nu e\, (\overline{e} b_n.\mathbf{0}\mid \,!\, e(z_1).c(y_1). c'(y'_1).\overline{z_1}\, y'_1.f(x). \overline{e} y_1.\mathbf{0} )\\
Y' &\triangleq & \nu g\, (\overline{g} b'_n.\,\mathbf{0}\mid \,!\, g(z_2).\,z_2(y_2).\, y_2(\chi_2).\, \overline{g} y_2.\,\mathbf{0} ).
\end{array}\]
Here, $(\verb"M"^{q_i,D^{\blacklozenge},B},\mathbf{v_i })$ is the pointed action model for passing $q_i$ from $D^{\blacklozenge}$ to $B$ with $1\leq i \leq n $.
\begin{remark}\label{rem:FIFO}
Applying the e-calculus, we can construct a \emph{chained} buffer pool with the heading cell $h$, its alias $h'$, and $n$ storage cells $b_1$, $\cdots$, $b_n$, their aliases $b'_1$, $\cdots$, $b'_n$ respectively, which store the facts $q_1$, $\cdots$, $q_n$ obtained from some external source respectively; moreover, the aliases $b'_1$, $\cdots$, $b'_n$ are stored in the cells $h$, $b_1$, $\cdots$, $b_{n-1}$ as \emph{links} respectively. The agents can read these facts through the \emph{cell aliases} from the buffer pool in the order \emph{FIFO}.
\end{remark}
%It is obvious that we may continue to construct new cells and read data from them, similarly.
%\end{example}
\subsection{Reading in an arbitrary order}\label{subsec-asynchronous communication-arbitrary order}
%\begin{example}[Unbounded unchained buffer pool read in arbitrary order]\label{example-asynchronous communication-arbitrary order}
This subsection will define an e-system as an example that  provides an unbounded buffer pool which can be read in an arbitrary order (i.e.,  \emph{unchained}).

%In this subsection, we still make the assumption that all the used names are \emph{private}.

At first glance, we may consider the following e-system $G$ where
\[\begin{array}{lll}
G& \triangleq &(\nu a,c)\,([\,Q\,]_{D^{\blacklozenge}} \parallel [\,X\,]_{D}\parallel \underset{1\leq j \leq m}{ \parallel} [\,Y\,]_{A_j}) \\
Q&\triangleq & !\,\nu b\, (\overline{c} b.\,b(\chi) .\, !\,\overline{a} b.\,\overline{b} \chi.\,\mathbf{0})\\
%P&\triangleq & !\,g(\chi).\, c(x).\,\overline{x} \chi.\,\mathbf{0}\\
Y &\triangleq & \,!\, a(y).\,y(\chi').\, \mathbf{0}\\
X & \triangleq & (\nu e,d)\,(\overline{e} d.\,\mathbf{0}\mid \,!\,e(z).\,a^{\ast}(\chi^{\ast}) .\,c(x).\,\overline{x} \chi^{\ast}.\,\overline{e} d.\,\mathbf{0}).
%X & \triangleq & !\,a^{\ast}(\chi^{\ast}) .\, c(x).\,\overline{x} \chi^{\ast}.\,\mathbf{0}.
%X & \triangleq & \underset{1\leq i \leq n}{ \LARGE{\mid}} c(x).\,\overline{x} q_i.\,\mathbf{0}.
\end{array}\]
with $m\geq 1$. Similar as in Subsection~\ref{subsec-asynchronous communication-FIFO}, the process $Q$ is a generator of storage cells, and the difference is that cells do not store the names (\emph{links}) of other cells, that is, the buffer pool is \emph{unchained}. The agent $D^{\blacklozenge} $ manages the buffer pool constructed  through $Q$. The agent $D$ obtains basic facts from an external source via the channel $a^{\ast}$ and stores these facts
 %$q_1$, $\cdots$, $q_n$
 into the buffer pool. Thus, the agents $A_1$, $\cdots$, $A_m$ may read these facts from the buffer pool in an arbitrary order.

 Nevertheless, this e-system $G$ has the  possibility of evolution below, which is undesired.\\

\noindent\textbf{(Flaw)} ~~Evidently, the combinations $!\,\overline{a} b.\,\overline{b} \chi$ and $!\, a(y).\,y(\chi')$ may interact repeatedly, which means that the agents  might read the \emph{same} cells \emph{repeatedly}.
By Proposition~\ref{prop:succ1}, Proposition~\ref{prop:succ2}, Proposition~\ref{prop:action execution preserves fact cognizability}, Proposition~\ref{prop:idempotent} and Proposition~\ref{lemma:commutativity}, this kind of repetition would not affect epistemic states. \\

\noindent\textbf{(Revise)} ~~In order for each agent to read each cell only once, in the following, we will revise $Q$ and $Y$ by providing each agent $A_j$ ($1\leq j \leq m$) with a \emph{special visiting channel} $a_j$ which will be opened for each cell only once. The revised parts are underlined to emphasize.
\[\begin{array}{lll}
G& \triangleq &(\nu \underline{a_1,\cdots,a_m},c)\,([Q]_{D^{\blacklozenge}} \parallel [X]_{D}\parallel \underset{1\leq j \leq m}{ \parallel} [\underline{Y_j}]_{A_j} )\\
Q&\triangleq & !\, \nu b\, \left(\overline{c} b.\,b(\chi) . \left(\underline{(\underset{1\leq j \leq m}{\mid}\overline{b} a_j.\,\mathbf{0})}\mid \, !\,\underline{b(y).\,\overline{y}\chi}.\,\mathbf{0}\right) \right)\\
%P&\triangleq & !\,g(\chi).\, c(x).\,\overline{x} \chi.\,\mathbf{0}\\
%I & \triangleq & \underset{1\leq i \leq n}{ \LARGE{\mid}} c(z).\,\overline{z} q_i.\,\mathbf{0}\\
\underline{Y_j} &\triangleq & !\, \underline{a_j(\chi_j)}.\, \mathbf{0}\\
X & \triangleq & (\nu e,d)\,(\overline{e} d.\,\mathbf{0}\mid \,!\,e(z).\,a^{\ast}(\chi^{\ast}) .\,c(x).\,\overline{x} \chi^{\ast}.\,\overline{e} d.\,\mathbf{0}).
%X & \triangleq & !\,a^{\ast}(\chi^{\ast}) .\, c(x).\,\overline{x} \chi^{\ast}.\,\mathbf{0}.
%X & \triangleq & \underset{1\leq i \leq n}{ \LARGE{\mid}} c(x).\,\overline{x} q_i.\,\mathbf{0}.
\end{array}\]
Here, in the process $Q$, the interaction between $\overline{b} a_j$ and $b(y)$ realizes that the channel $a_j$ is opened for the cell $b$ and then $A_j$ can read from the cell $b$ via $a_j$.
%\end{example}
\begin{remark}\label{rem:arbitrary}
Applying the e-calculus, we can construct a \emph{unchained} buffer pool with its storage cells provided with a special visiting channel. The agents can read from each cell through its special visiting channel in an \emph{arbitrary order}.
\end{remark}
\section{Discussion}\label{sec:discussion}
Based on the classical $\pi$-calculus, we have presented the e-calculus to deal with epistemic interactions in the concurrency situations. It has no doubt that the e-calculus may formalize synchronous communication between two agents. We also adopts the e-calculus to specify intelligent multi-agent systems associated with epistemic attitudes in asynchronous situations. Unlike dynamic epistemic logics, the e-calculus focuses on the operational semantics of e-systems referring to epistemic interactions instead of static logic laws concerning epistemic actions (e.g., public announcement, private announcement, asynchronous communication, etc). On the other hand, through applying AML to describe epistemic states, the e-calculus can indirectly reflects logical specifications of AML and so describe indirectly the logical laws of epistemic effects caused by epistemic interactions.

In~\citep{Knight2017reasoningaboutknowledgeinasynchronoussystem}, in order to illustrate static logic laws for asynchronous communication, each agent is provided with a \emph{private} FIFO channel, which receives the messages (formulas) announced publicly from one trusted public source. The  literature~\citep{Ditmarsch2017asynchronousannouncement} and~\citep{Balbiani2021asynchronousannouncement} made the similar assumption. These three papers focus on receivers and their cognition changes. In this paper, it is \emph{between agents} that the messages (propositional letters) are passed. Subsection~\ref{subsec-asynchronous communication-FIFO} (\ref{subsec-asynchronous communication-arbitrary order}) constructs a buffer pool  storing these propositional letters which can be read in FIFO order (resp., in an arbitrary order). Such a buffer pool  may be  \emph{shared} by all agents.

In~\citep{Knight2017reasoningaboutknowledgeinasynchronoussystem,Ditmarsch2017asynchronousannouncement,Balbiani2021asynchronousannouncement}, \emph{truthful} announcements are publicly sent. Due to this, certain (invalid) message sequences can be ruled out, just like \emph{inconsistent cuts} in distributed computing. A \emph{cut} is a list stating which announcements have been received by which agents.
This paper, as usual, makes the assumption that an agent can only send the messages that it \emph{knows} at the current epistemic state (see, the SOS rule (OUT$_{\text{e}}$)). Through stating whether $M,s\models \Box_A q $ holds, we may determine whether the agent $A$ has the capability for sending the fact $q$ or not at the epistemic state $(M,s)$. It is evident that this is more concise and intuitive.

\bibliographystyle{plain}
\bibliography{cei}

\end{document}